\documentclass[11pt]{article}
\usepackage{amssymb, amsmath, amsthm}

\usepackage{amsfonts,mathtools,mathrsfs,mathtools,color}
\usepackage[colorlinks]{hyperref}
\usepackage{graphicx}
\usepackage{thmtools}

\newtheorem{theorem}{Theorem}
\newtheorem*{theorem*}{Theorem}
\newtheorem{lemma}{Lemma}
\newtheorem{corollary}{Corollary}
\newtheorem{example}{Example}

\newcommand{\R}{\mathbb{R}} 
\newcommand{\N}{\mathcal{N}}

\newcommand{\E}{\mathbb{E}}

\renewcommand{\part}[2]{\frac{\partial #1}{\partial #2}}

\newcommand{\Tr}{\mathrm{Tr}}
\newcommand{\Var}{\mathrm{Var}}

\newcommand{\HS}{\mathrm{HS}}

\newcommand{\Lg}{L}
\newcommand{\Lh}{M}
\newcommand{\op}{\mathrm{op}}
\newcommand{\PLA}{\mathrm{PLA}}
\newcommand{\ULA}{\mathrm{ULA}}

\usepackage{array}
\newcolumntype{C}{>{\centering\arraybackslash}m{.25\linewidth}}
\newcolumntype{D}{>{\centering\arraybackslash}m{.24\linewidth}}
\newcolumntype{E}{>{\centering\arraybackslash}m{.18\linewidth}}
\newcolumntype{F}{>{\centering\arraybackslash}m{.18\linewidth}}

\title{Proximal Langevin Algorithm: \\ Rapid Convergence Under Isoperimetry}
\author{Andre Wibisono \\ 
College of Computing \\
Georgia Institute of Technology \\
\texttt{wibisono@gatech.edu}}

\begin{document}

\maketitle

\begin{abstract}
We study the Proximal Langevin Algorithm (PLA) for sampling from a probability distribution $\nu = e^{-f}$ on $\R^n$ under isoperimetry.
We prove a convergence guarantee for PLA in Kullback-Leibler (KL) divergence when $\nu$ satisfies log-Sobolev inequality (LSI) and $f$ has bounded second and third derivatives.
This improves on the result for the Unadjusted Langevin Algorithm (ULA), and matches the fastest known rate for sampling under LSI (without Metropolis filter) with a better dependence on the LSI constant.
We also prove convergence guarantees for PLA in R\'enyi divergence of order $q > 1$ when the biased limit satisfies either LSI or Poincar\'e inequality.
\end{abstract}

\section{Introduction}

Sampling is a fundamental algorithmic task.
While the case of logconcave sampling is relatively well-studied, there are recent efforts in understanding the convergence guarantees for non-logconcave sampling.
This is motivated by practical applications which require sampling complicated distributions in high-dimensional spaces, as well as by the recent progress in non-convex optimization.

In this paper we study the {\em Proximal Langevin Algorithm (PLA)} for sampling from a probability distribution $\nu = e^{-f}$ on $\R^n$:
\begin{align*}
x_{k+1} = x_k - \epsilon \nabla f(x_{k+1}) + \sqrt{2\epsilon} z_k
\end{align*}
where $\epsilon > 0$ is step size and $z_k \sim \N(0,I)$ is an independent Gaussian random variable in $\R^n$.
The above is an implicit update, and we assume we can solve for $x_{k+1}$, for example via the proximal step (see Section~\ref{Sec:PLA} for more detail):
\begin{align*}
x_{k+1} = \arg\min_{x \in \R^n} \left\{ f(x) + \frac{1}{2\epsilon} \|x-(x_k + \sqrt{2\epsilon} z_k)\|^2 \right\}.
\end{align*}

PLA is a discretization of the continuous-time Langevin dynamics that uses the backward (implicit) method for the gradient.
It is an implicit variant of the Unadjusted Langevin Algorithm (ULA), which uses the forward (explicit) method for the gradient.
PLA was introduced by Pereyra~\cite{P16} (in a more general form with Metropolis filter) from a smoothing perspective, and it was also proposed by Bernton~\cite{B18} from an optimization perspective of sampling.
PLA has been widely applied in practice, in particular when combined with ULA to sample from composite distributions~\cite{DMP18,CPM18,RPW19,CKCetal19}, and analyzed under logconcavity or strong logconcavity~\cite{B18, BDMP17, BSS19, SKR19}.
Analogous to implicit vs.\ explicit methods in optimization, we expect PLA to have better properties than ULA at the cost of a more expensive per-iteration complexity (solving an optimization problem).
See also~\cite{HSR19, VPZ19} for some recent extensions of PLA.

In this paper we prove convergence guarantees for PLA under {\em isoperimetry}, namely, when the target distribution $\nu$ satisfies either log-Sobolev inequality (LSI) or Poincar\'e inequality.
Isoperimetry is a natural relaxation of logconcavity that still allows for fast sampling in continuous time.
Strong logconcavity (SLC) implies LSI, and in turn implies Poincar\'e inequality with the same constant.
Moreover, logconcavity and bounded diameter implies LSI and Poincar\'e inequality.
However, isoperimetry is more general, as it is preserved under Lipschitz mapping and bounded perturbation, whereas logconcavity is not.
Therefore, there is a wide class of probability distributions, including multimodal ones, satisfying isoperimetry.

In continuous time, isoperimetry is sufficient for fast sampling.
For example, if $\nu$ satisfies LSI with constant $\alpha > 0$, then along the Langevin dynamics, the Kullback-Leibler (KL) divergence converges exponentially fast at rate $2\alpha$. 
This means the Langevin dynamics reaches KL divergence less than $\delta$ in time $t = O(\frac{1}{\alpha}\log \frac{1}{\delta})$.
In particular, there is no dependence on dimension and no assumption on the smoothness of $\nu$ is required, beyond differentiability in order to run the Langevin dynamics.
This is analogous to the exponential convergence of gradient flow for optimization in continuous time under gradient domination condition (via the perspective of sampling as optimization in the space of measures~\cite{JKO98,W18}).

In discrete time, sampling is more challenging.
We can discretize continuous-time dynamics to obtain algorithms, such as PLA or ULA above from the Langevin dynamics.
We need some smoothness assumptions (bounds on derivatives of $\nu$) to control the discretization error, so the iteration complexity now depends on the condition number.
However, the discretization error leads to an asymptotic {\em bias}, which means the algorithm converges to the wrong distribution.
This bias arises because in algorithms such as PLA or ULA we are applying mismatched splitting methods for solving a composite optimization problem in the space of measures; see~\cite{W18} for more discussion.

It is possible to remove the bias by applying the Metropolis filter (accept-reject step) in each iteration; this has a geometric interpretation as projection in total variation (TV) distance~\cite{BD01}.
With the Metropolis filter, it is possible to prove the algorithm still converges exponentially fast in discrete time, and obtain an iteration complexity of $O(\log \frac{1}{\delta})$ to reach error $\delta$ in TV distance with warm start and under various conditions such as strong logconcavity, isoperimetry, or distant dissipativity~\cite{BH13,DCWY18,MV19,MFWB19b}.
However, if we want convergence in KL divergence---which is stronger---then Metropolis filter does not work because it makes the distributions singular (have point masses).
Furthermore, Metropolis filter can slow down the algorithm in practice when the rejection probability is high.

In this paper we follow another approach, which is to control the convergence of the algorithm and the size of the bias, then choose a small enough step size to make the error less than any given threshold.
This approach was pioneered by Dalalyan~\cite{D17RSS,D17} and Durmus and Moulines~\cite{DM17} to analyze ULA under strong logconcavity, and has been extended to many other algorithms. 
However, the bias becomes a bottleneck in complexity.
The bias scales with some power of the step size $\epsilon$, resulting in an iteration complexity which is polynomial in $\frac{1}{\delta}$ (rather than logarithmic as in continuous time) to reach error $\delta$ in KL divergence.
For example, we show in~\cite{VW19} that under LSI and second-order smoothness, the bias of ULA  is $O(\epsilon)$, resulting in an iteration complexity of $\tilde O(\frac{1}{\delta})$ (ignoring dimension dependence for now).
However, basic considerations suggest the correct bias is $O(\epsilon^2)$ since ULA and PLA are first-order discretization, which will yield an iteration complexity of $\tilde O(\frac{1}{\sqrt{\delta}})$.
In this paper we show this is indeed the case for PLA under LSI and third-order smoothness.

Our main result is the following.
We say $\nu = e^{-f}$ is $(\Lg,\Lh)$-smooth if $\|\nabla^2 f\| \le \Lg$ and $\|\nabla^3 f\| \le \Lh$.
Here $H_\nu(\rho)$ is the KL divergence of $\rho$ with respect to $\nu$.
See Theorem~\ref{Thm:Main} in Section~\ref{Sec:PLALSI} for detail.

\begin{restatable*}{theorem}{ThmMain}\label{Thm:Main}
Assume $\nu$ satisfies $\alpha$-LSI and is $(\Lg,\Lh)$-smooth.
For any $x_0 \sim \rho_0$, the iterates $x_k \sim \rho_k$ of PLA with step size $0 < \epsilon \le \min\{ \frac{1}{8\Lg}, \frac{1}{\Lh}, \frac{3\alpha}{32 \Lg^2} \}$ satisfies:
\begin{align}\label{Eq:MainBound}
H_\nu(\rho_k) \le e^{-\alpha \epsilon k} H_\nu(\rho_0) + \frac{34 \, \epsilon^2 n(\Lg^3 + 9n^2 \Lh^2)}{\alpha}.
\end{align}
\end{restatable*}

This implies the following iteration complexity for PLA under LSI:
to reach $H_\nu(\rho_k) \le \delta$, it suffices to run PLA with $\rho_0 = \N(x^\ast,\frac{1}{L} I)$ and step size $\epsilon  = \Theta\left(\sqrt{\frac{\alpha \delta}{n(\Lg^3+n^2\Lh^2)}}\right)$ for
\begin{align}\label{Eq:PLAComp}
k 
= \tilde O \left( \frac{n^{\frac{1}{2}}(\Lg^{\frac{3}{2}} + n\Lh)}{\alpha^{\frac{3}{2}} \delta^{\frac{1}{2}}} \right)
\end{align}
iterations.
Here $x^\ast$ is a stationary point of $f$ ($\nabla f(x^\ast) = 0$), which we can find via gradient descent.

This improves on the result~\cite{VW19} for ULA, in which we show under $\alpha$-LSI and $(L,\infty)$-smoothness, ULA has iteration complexity $k = \tilde O\left( \frac{n \Lg^2}{\alpha^{2} \delta}\right)$.
However, as noted above, it is likely the analysis in~\cite{VW19} is not tight since it only implies a bias of $O(\epsilon)$ for ULA rather than $O(\epsilon^2)$ for PLA in Theorem~\ref{Thm:Main}.
We prove Theorem~\ref{Thm:Main} by comparing a continuous-time interpolation of PLA with the Langevin dynamics to establish a recurrence for the decrease of KL divergence in each iteration; this technique is similar to~\cite{VW19} and earlier papers~\cite{D17,CB18}.
Our improvement comes because we can show a tight error bound for the interpolation of PLA by comparing it with the weighted Langevin dynamics; see Section~\ref{Sec:ThmMainProof}.
Furthermore, we illustrate in the Gaussian case that the bias is indeed $\Theta(\epsilon^2)$.

\begin{table}[h!t!b!p!]
\begin{center}
  \begin{tabular}{ | C | D | E | F |}
    \hline
    {\bf Algorithm} & {\bf Assumptions} & {\bf Iterations to $H_\nu(\rho_k) \le \delta$} & {\bf Iterations to $W_2(\rho_k,\nu) \le \delta$} \\ \hline 
    \hline
    ULA~\cite{VW19} & $\alpha$-LSI, $(L,\infty)$-smoothness & $\tilde O\Big( \frac{n \Lg^2}{\alpha^{2} \delta}\Big)$ & $\tilde O\Big( \frac{n \Lg^2}{\alpha^{3} \delta^2}\Big)$ \\ \hline
    Underdamped Langevin dynamics~\cite{Ma19} & $\alpha$-LSI, $(L,M)$-smoothness & $\tilde O\Big(\frac{n^{\frac{1}{2}}(\Lg^{\frac{3}{2}} + n^{\frac{1}{2}} \Lh)}{\alpha^2 \delta^{\frac{1}{2}}}\Big)$ & $\tilde O\Big(\frac{n^{\frac{1}{2}}(\Lg^{\frac{3}{2}} + n^{\frac{1}{2}} \Lh)}{\alpha^{\frac{5}{2}} \delta}\Big)$ \\ \hline
    Randomized midpoint for ULD~\cite{SL19} & $\alpha$-SLC, $(L,\infty)$-smoothness & - & $\tilde O\Big( \frac{n^{\frac{1}{3}} L}{\alpha^{\frac{4}{3}} \delta^{\frac{2}{3}}} \Big)$ \\ \hline
    PLA (this paper) & $\alpha$-LSI, $(L,M)$-smoothness & $\tilde O \Big( \frac{n^{\frac{1}{2}}(\Lg^{\frac{3}{2}} + n\Lh)}{\alpha^{\frac{3}{2}} \delta^{\frac{1}{2}}} \Big)$ & $\tilde O \Big( \frac{n^{\frac{1}{2}}(\Lg^{\frac{3}{2}} + n\Lh)}{\alpha^2 \delta} \Big)$ \\
    \hline
  \end{tabular}
  \caption{Iteration complexities for Langevin algorithms under LSI, and the fastest under SLC. Here $H_\nu(\rho)$ is the KL divergence and $W_2(\rho,\nu)$ is the Wasserstein distance.}
\end{center}
\end{table}

We note a recent work~\cite{MFWB19} improves the analysis of ULA under LSI and third-order smoothness with an additional dissipativity assumption, and shows an iteration complexity for ULA which is similar to our result~\eqref{Eq:PLAComp} for PLA.

Currently the fastest (in terms of error $\delta$ in KL divergence) algorithm for sampling under LSI is a discretization of the underdamped Langevin dynamics~\cite{Ma19}, which has iteration complexity $k = \tilde O\left(\frac{n^{1/2}(\Lg^{3/2} + n^{1/2} \Lh)}{\alpha^2 \delta^{1/2}}\right)$ under $\alpha$-LSI and $(L,M)$-smoothness to reach $H_\nu(\rho_k) \le \delta$.
We see from~\eqref{Eq:PLAComp} that PLA has the same dependence on $\delta$ but better dependence on $\alpha$.

We recall LSI implies Talagrand's inequality, which bounds Wasserstein distance by KL divergence $W_2(\rho,\nu)^2 \le \frac{2}{\alpha} H_\nu(\rho)$.
Then Theorem~\ref{Thm:Main} also implies the iteration complexity for PLA to reach $W_2(\rho_k,\nu) \le \delta$ under $\alpha$-LSI and $(\Lg,\Lh)$-smoothness is
\begin{align*}
k = \tilde O \left( \frac{n^{\frac{1}{2}}(\Lg^{\frac{3}{2}} + n\Lh)}{\alpha^2 \delta} \right).
\end{align*}
A previous analysis~\cite{B18} shows an iteration complexity of $k = \tilde O\Big(\frac{n\Lg}{\alpha \delta^2}\Big)$ for PLA to reach $W_2(\rho_k,\nu) \le \delta$ under $\alpha$-SLC and $(\Lg,\infty)$-smoothness.
Thus, our result shows a better iteration complexity for PLA under SLC and third-order smoothness.  

We note for sampling under SLC, faster rates are achieveable via more advanced algorithms, whose analyses are made possible by coupling techniques.
Currently the fastest algorithm is a randomized midpoint discretization of the underdamped Langevin dynamics~\cite{SL19}, which has iteration complexity $k = \tilde O\left( \frac{n^{1/3} L}{\alpha^{4/3} \delta^{2/3}} \right)$ to reach $W_2(\rho_k,\nu) \le \delta$ under $\alpha$-SLC and $(\Lg,\infty)$-smoothness, and it can be made faster by parallelizing.
See also~\cite{MMWBJ19} for a higher-order Langevin dynamics that achieves a similar iteration complexity under an additional separability assumption.
Previously the fastest results were by Hamiltonian Monte Carlo~\cite{LV18,CV19,MS17,MV18}, various discretization of the overdamped or underdamped Langevin dynamics~\cite{DR18,DM19,CB18,D17,DK19}, or using higher-order integrators such as stochastic Runge-Kutta~\cite{LWME19}. 
Thus, there is a gap between the known complexity for sampling under LSI and under SLC.
It is interesting to understand whether these more advanced algorithms can be analyzed under LSI, when coupling techniques no longer work.

We also note that for the case when $\nu$ is logconcave, there are other methods that can be used, including the ball walk and hit-and-run \cite{KLS97, LV07, LV06, LV06b}, which have iteration complexity with logarithmic dependence on the error $\delta$ in TV distance or $\chi^2$-divergence, and no dependence on the condition number.

Our second main result is a convergence guarantee for R\'enyi divergence of order $q > 1$ along PLA when the biased limit satisfies either LSI or Poincar\'e inequality.
R\'enyi divergence of order $q > 1$ is a stronger generalization of KL divergence (which is the case $q=1$) with fundamental applications in statistics, physics, and computer science~\cite{R61,VH14,BCG19,DR16}.
Under LSI, R\'enyi divergence converges exponentially fast along the Langevin dynamics.
Under Poincar\'e inequality, R\'enyi divergence still converges along the Langevin dynamics, but now at a rate which is initially linear, then exponential.
We show that when the biased limit $\nu_\epsilon$ of PLA satisfies either LSI or Poincar\'e inequality, R\'enyi divergence with respect to $\nu_\epsilon$ converges along PLA at the same speed as along the Langevin dynamics.
We can combine this with a decomposition property of R\'enyi divergence to obtain an iteration complexity for PLA in R\'enyi divergence which is controlled by the size of the bias; see Theorem~\ref{Thm:RenyiRate} in Section~\ref{Sec:PLARenyi} and Theorem~\ref{Thm:RenyiRatePoincare} in Section~\ref{Sec:PLARenyiPoincare}.
Furthermore, the iteration complexity under Poincar\'e inequality is a factor of $n$ larger than the complexity under LSI.
These results are similar to the result~\cite{VW19} for ULA.
However, we illustrate with an example in the Gaussian case that the bias in R\'enyi divergence of PLA is smaller (and always finite) than the bias of ULA (which can be infinite).

The rest of this paper is organized as follows.
In Section~\ref{Sec:PLA} we state the algorithm and main result on the convergence of KL divergence under LSI.
In Section~\ref{Sec:ResultRenyi} we state the second result on the convergence of R\'enyi divergence under LSI or Poincar\'e inequality.
In Section~\ref{Sec:Langevin} we review the Langevin dynamics.
In Sections~\ref{Sec:ProofPLA} and~\ref{Sec:ProofRenyi} we provide proofs and details.
We conclude with a discussion in Section~\ref{Sec:Disc}.

\section{Algorithm and main result}
\label{Sec:PLA}

Let $\nu = e^{-f}$ be the target probability distribution on $\R^n$.
We assume $f \colon \R^n \to \R$ is differentiable.

\subsection{Proximal Langevin Algorithm}

We study the {\bf Proximal Langevin Algorithm (PLA)} that starts
from any random variable $x_0  \in \R^n$ and maintains the iterates
\begin{align}\label{Eq:PLA}
x_{k+1} = x_k - \epsilon \nabla f(x_{k+1}) + \sqrt{2\epsilon} z_k
\end{align}
where $\epsilon > 0$ is step size and $z_k \sim \N(0,I)$ is an independent Gaussian random variable.
The above is an implicit update, and we assume we can solve for $x_{k+1}$, for example via the proximal step:
\begin{align}\label{Eq:PLA2}
x_{k+1} = \arg\min_{x \in \R^n} \left\{ f(x) + \frac{1}{2\epsilon} \|x-(x_k + \sqrt{2\epsilon} z_k)\|^2 \right\}.
\end{align}
Indeed, the solution of~\eqref{Eq:PLA2} satisfies $\nabla f(x_{k+1}) + \frac{1}{\epsilon}(x_{k+1}-(x_k + \sqrt{2\epsilon} z_k)) = 0$, which is~\eqref{Eq:PLA}.
Note that the formulation~\eqref{Eq:PLA2} also makes sense when $f$ is not differentiable.
If $f$ is $(1/\epsilon)$-smooth ($\|\nabla^2 f\| \le \frac{1}{\epsilon}$), then~\eqref{Eq:PLA2} is a strongly convex optimization problem with a unique minimizer $x_{k+1}$, so PLA is well-defined.
If $f$ is convex, then the restriction $\epsilon \le \frac{1}{\|\nabla^2 f\|}$ can be removed. 

\begin{example}
Let $\nu = \N(0,\Sigma)$ be Gaussian with mean $0$ and covariance $\Sigma \succ 0$, so $\nabla f(x) = \Sigma^{-1} x$.
The PLA iteration is $x_{k+1} = A(x_k + \sqrt{2\epsilon} z_k)$, so $x_k \stackrel{d}{=} A^k x_0 + \sqrt{2\epsilon} A (I-A^2)^{-\frac{1}{2}} (I - A^{2k})^{\frac{1}{2}} \tilde z_k$ where $A = (I+\epsilon \Sigma^{-1})^{-1}$ and  $\tilde z_k \sim \N(0,I)$ is independent.
Note that for any $\epsilon > 0$, $A^k \to 0$ as $k \to \infty$.
Thus, for any $\epsilon > 0$, PLA converges to $\nu_\epsilon = \N(0,\Sigma_\epsilon)$ where $\Sigma_\epsilon = 2\epsilon A^{2}(I-A^2)^{-1}  
= \Sigma(I+\frac{\epsilon}{2}\Sigma^{-1})^{-1}$.
\end{example}

PLA (in a more general form with a Metropolis filter) was first introduced in~\cite{P16} from a smoothing perspective.
PLA was also studied in~\cite{B18} from the optimization perspective of sampling under logconcavity assumption.\footnote{Both~\cite{P16,B18} apply the proximal step before the Gaussian step in each iteration, while PLA applies them in the opposite order. Over $k$ iterations, both~\cite{P16,B18} and PLA only differ by a single proximal or Gaussian step.}
PLA is the implicit variant of another popular algorithm, the Unadjusted Langevin Algorithm (ULA), which is the explicit iteration $x_{k+1} = x_k - \epsilon \nabla f(x_k) + \sqrt{2\epsilon} z_k$.
PLA and ULA are discretization of the Langevin dynamics in continuous time (see Section~\ref{Sec:Langevin} for a review), where PLA applies the backward (implicit) method to discretize the gradient, while ULA applies the forward (explicit) method.
This makes PLA more expensive to implement in practice, but it offers better behavior and analysis than ULA.
For example, in the Gaussian case $\nu = \N(0,\Sigma)$, recall ULA converges to $\nu_\epsilon^{\text{ULA}} = \N(0,\Sigma_\epsilon^{\text{ULA}})$ only for $\epsilon < 2\|\Sigma^{-1}\|^{-1}$, where $\Sigma_\epsilon^{\text{ULA}} = \Sigma(I-\frac{\epsilon}{2}\Sigma^{-1})^{-1}$; see also~\cite{W18}.
In this case PLA always converges and the bias is smaller than the bias of ULA.

\begin{example}[PLA vs.\ ULA for Gaussian]
Let $\nu = \N(0,\Sigma)$.
For $\epsilon > 0$, the limit of PLA is $\nu_\epsilon^{\text{PLA}} = \N(0,\Sigma(I+\frac{\epsilon}{2}\Sigma^{-1})^{-1})$.
For $0 < \epsilon < 2\|\Sigma^{-1}\|^{-1}$, the limit of ULA is $\nu_\epsilon^{\text{ULA}} = \N(0,\Sigma(I-\frac{\epsilon}{2}\Sigma^{-1})^{-1})$.
Let $\lambda_1,\dots,\lambda_n > 0$ denote the eigenvalues of $\Sigma$.
The bias of PLA in relative entropy is
\begin{align*}
H_\nu(\nu_\epsilon^{\PLA}) 
= \frac{1}{2} \left(\Tr\left( I + \frac{\epsilon}{2} \Sigma^{-1} \right)  - n - \log \det \left(I + \frac{\epsilon}{2} \Sigma^{-1}\right) \right)
= \frac{1}{2} \sum_{i=1}^n \left( \frac{\epsilon}{2\lambda_i} - \log \left(1 + \frac{\epsilon}{2\lambda_i} \right) \right)
\end{align*}
while the bias of ULA is
\begin{align*}
H_\nu(\nu_\epsilon^{\ULA}) 
= \frac{1}{2} \left(\Tr\left( I - \frac{\epsilon}{2} \Sigma^{-1} \right)  - n - \log \det \left(I - \frac{\epsilon}{2} \Sigma^{-1}\right) \right)
= \frac{1}{2} \sum_{i=1}^n \left( -\frac{\epsilon}{2\lambda_i} - \log \left(1 - \frac{\epsilon}{2\lambda_i} \right) \right).
\end{align*}
Note that we always have
\begin{align*}
H_\nu(\nu_\epsilon^{\PLA}) < H_\nu(\nu_\epsilon^{\ULA}).
\end{align*}
Furthermore, 
$H_\nu(\nu_\epsilon^{\PLA}) = \frac{1}{16} \sum_{i=1}^n \big( \frac{\epsilon^2}{\lambda_i^2} - \frac{\epsilon^3}{3\lambda_i^3} \big) + O(\epsilon^4)$ and
$H_\nu(\nu_\epsilon^{\ULA}) = \frac{1}{16} \sum_{i=1}^n \big( \frac{\epsilon^2}{\lambda_i^2} + \frac{\epsilon^3}{3\lambda_i^3} \big)+O(\epsilon^4)$.
\end{example}

\subsection{Log-Sobolev inequality and smoothness assumption}

Before stating our results, we recall some definitions for the analysis.

\subsubsection{Log-Sobolev inequality}

Let $\rho, \nu$ be probability distributions on $\R^n$ with smooth densities and finite second moments.
We recall the {\em relative entropy} (or Kullback-Leibler (KL) divergence) of $\rho$ with respect to $\nu$ is
\begin{align}
H_\nu(\rho) = \int_{\R^n} \rho(x) \log \frac{\rho(x)}{\nu(x)} \, dx.
\end{align}
Relative entropy has the property that $H_\nu(\rho) \ge 0$, and $H_\nu(\rho) = 0$ if and only if $\rho = \nu$.
The {\em relative Fisher information} of $\rho$ with respect to $\nu$ is
\begin{align}\label{Eq:J}
J_\nu(\rho) = \int_{\R^n} \rho(x) \left\|\nabla \log \frac{\rho(x)}{\nu(x)}\right\|^2 dx.
\end{align}
The geometric meaning of relative Fisher information is as the squared gradient of relative entropy in the space of probability measures with the Wasserstein metric.

We recall $\nu$ satisfies {\em log-Sobolev inequality} (LSI) with constant $\alpha > 0$ if for all $\rho$,
\begin{align}\label{Eq:LSI}
H_\nu(\rho) \le \frac{1}{2\alpha} J_\nu(\rho).
\end{align}
LSI has the geometric interpretation as the gradient domination condition for relative entropy in the Wasserstein space~\cite{OV00}, which ensures the Langevin dynamics converges exponentially fast in continuous time; see Section~\ref{Sec:Langevin} for a review.
We recall $\nu = e^{-f}$ is {\em strongly logconcave} (SLC) with constant $\alpha > 0$ if $f$ is $\alpha$-strongly convex: $\nabla^2 f(x) \succeq \alpha I$ for all $x \in \R^n$.
SLC is a strong condition that allows the analysis of many sampling algorithms.
However, SLC is brittle, as it is not preserved under perturbation or arbitrary mapping.
A classic result by Bakry and \'Emery~\cite{BE85} shows that SLC implies LSI with the same constant $\alpha$.
Furthermore, LSI is more stable, as it is preserved under bounded perturbation and Lipschitz mapping.
LSI also has an isoperimetric content as a bound on the log-Cheeger constant, see for example~\cite{L94}.
Therefore, LSI provides a natural condition to obtain fast sampling in discrete time.

\subsubsection{Smoothness assumptions}

We say $\nu = e^{-f}$ is {\em $(\Lg,\Lh)$-smooth} if $f$ is three-times differentiable and satisfies the following two conditions:
\begin{enumerate}
  \item The gradient $\nabla f$ is $\Lg$-Lipschitz:
  \begin{align*}
  \|\nabla f(x) - \nabla f(y)\| \le \Lg \|x-y\| ~~~~ \text{ for all } \, x,y \in \R^n,
  \end{align*}
  or equivalently, $\|\nabla^2 f(x)\|_\text{op} \le \Lg$, which means $-LI \preceq \nabla^2 f(x) \preceq LI$.
  
  \item The Hessian $\nabla^2 f$ is $\Lh$-Lipschitz in the operator norm:
  \begin{align*}
  \|\nabla^2 f(x) - \nabla^2 f(y)\|_\op \le \Lh \|x-y\| ~~~~ \text{ for all } \, x,y \in \R^n.
  \end{align*}
  This implies $\|\nabla_i \nabla^2 f(x)\|_\op \le \Lh$ for $i = 1,\dots,n$, where $\nabla_i \nabla^2 f(x)$ is the matrix with $(j,k)$ entry $\part{}{x_i} (\nabla^2 f(x))_{jk} = \part{^3 f(x)}{x_i \, \partial x_j \, \partial x_k}$.
\end{enumerate}

\subsection{Main result: Convergence of relative entropy along PLA under LSI}
\label{Sec:PLALSI}

Our first main result is the following convergence guarantee of relative entropy along PLA when the target distribution $\nu$ satisfies LSI and a smoothness assumption.
We note that smoothness is only used in the analysis and not for the definition of PLA.

Here $\rho_k$ is the distribution of $x_k$ along PLA.
We provide the proof of Theorem~\ref{Thm:Main} in Section~\ref{Sec:ThmMainProof}.

\ThmMain

As $k \to \infty$, this implies the bias of PLA is $H_\nu(\nu_\epsilon) \le \frac{34 \, \epsilon^2 n(\Lg^3 + 9n^2 \Lh^2)}{\alpha} = O(\epsilon^2)$.
This bias is of the right order, since for Gaussian we have $H_\nu(\nu_\epsilon) = \Theta(\epsilon^2)$.
This is smaller than the bias for ULA from~\cite{VW19}, and thus yields a faster iteration complexity for PLA.

Concretely, given $\delta > 0$, to reach error $H_\nu(\rho_k) \le \delta$, it suffices to run PLA such that the two terms in~\eqref{Eq:MainBound} are each less than $\frac{\delta}{2}$:
So we want to run PLA with step size $\epsilon \le \sqrt{\frac{\alpha \delta}{68 \, n (\Lg^3 + 9n^2 \Lh^2)}}$ for $k \ge \frac{1}{\alpha \epsilon} \log \left(\frac{2 H_\nu(\rho_0)}{\delta}\right)$ iterations.
If we start with a Gaussian $\rho_0 = \N(x^\ast,\frac{1}{L} I)$ where $x^\ast$ is a stationary point ($\nabla f(x^\ast) = 0$, which we can find via gradient descent), then $H_\nu(\rho_0) \le f(x^\ast) + \frac{n}{2} \log \frac{L}{2\pi} = \tilde O(n)$, see~\cite[Lemma~1]{VW19}.
Therefore, Theorem~\ref{Thm:Main} implies the following iteration complexity for PLA.

\begin{corollary}\label{Cor:PLA}
Assume $\nu$ satisfies $\alpha$-LSI and is $(\Lg,\Lh)$-smooth.
To reach $H_\nu(\rho_k) \le \delta$, it suffices to run PLA with $\rho_0 = \N(x^\ast,\frac{1}{L} I)$ and step size $\epsilon  = \Theta\left(\sqrt{\frac{\alpha \delta}{n(\Lg^3+n^2\Lh^2)}}\right)$ for
\begin{align}
k = \tilde O \left( \frac{1}{\alpha \epsilon} \right)
= \tilde O \left( \frac{n^{\frac{1}{2}}(\Lg^{\frac{3}{2}} + n\Lh)}{\alpha^{\frac{3}{2}} \delta^{\frac{1}{2}}} \right)
\end{align}
iterations.
\end{corollary}

This matches the best known rate (in terms of $\delta$) for sampling under LSI, achieved by the underdamped Langevin algorithm~\cite{Ma19}, but PLA has better dependence on the LSI constant $\alpha$.

Furthermore, since LSI implies Talagrand's inequality ($H_\nu(\rho) \ge \frac{\alpha}{2} W_2(\rho,\nu)^2$) with the same constant~\cite{OV00},
Theorem~\ref{Thm:Main} also implies the iteration complexity for PLA to reach $W_2(\rho_k,\nu) \le \delta$ under $\alpha$-LSI and $(\Lg,\Lh)$-smoothness is
\begin{align}
k = \tilde O \left( \frac{n^{\frac{1}{2}}(\Lg^{\frac{3}{2}} + n\Lh)}{\alpha^2 \delta} \right).
\end{align}

\subsection{Analysis of relative entropy in one step of PLA}

The proof of Theorem~\ref{Thm:Main} relies on the following result which says relative entropy decreases by a constant factor with an additional $O(\epsilon^3)$ error term in each step of PLA; this leads to $O(\epsilon^2)$ bias for PLA as stated in Theorem~\ref{Thm:Main}.
In contrast, recall the analogous result for ULA~\cite[Lemma~3]{VW19} has $O(\epsilon^2)$ error in each iteration, which leads to $O(\epsilon)$ bias for ULA.

In the following, $\rho_k$ is the probability distribution of the iterates $x_k$ of PLA.
We provide the proof of Lemma~\ref{Lem:OneStep} in Section~\ref{Sec:ProofLemOneStep}.

\begin{lemma}\label{Lem:OneStep}
Assume $\nu$ satisfies $\alpha$-LSI and is $(\Lg,\Lh)$-smooth, and $0 < \epsilon \le \min\{ \frac{1}{8\Lg}, \frac{1}{\Lh}, \frac{3\alpha}{32 \Lg^2} \}$.
In each step of PLA, we have
\begin{align}\label{Eq:OneStep}
H_\nu(\rho_{k+1}) \le e^{-\alpha \epsilon} H_\nu(\rho_k) + 32 \epsilon^3 n(\Lg^3 + 9 n^2 \Lh^2).
\end{align}
\end{lemma}
\begin{proof}[Proof sketch]
The output $x_{k+1}$ of PLA~\eqref{Eq:PLA} is the value at time $t = \epsilon$ of the stochastic process
\begin{align}\label{Eq:Stoch10}
X_t = X_0 - t \nabla f(X_t) + \sqrt{2} W_t
\end{align}
starting at $X_0 = x_k$, where $W_t$ is the standard Brownian motion in $\R^n$.
We show in Lemma~\ref{Lem:SDERep} that $(X_t)_{t \ge 0}$~\eqref{Eq:Stoch10} evolves following the SDE
\begin{align}\label{Eq:SDE0}
dX_t = \mu \, dt + \sqrt{2 G} \, dW_t
\end{align}
where $\mu = -\sqrt{G} \left( \nabla f(X_t) + t \, \Tr(\nabla^3 f(X_t) G)\right)$ and $G = (I + t \nabla^2 f(X_t))^{-2}$.
Recall the Langevin dynamics with covariance $G$ converges to $\nu = e^{-f}$ if the drift is $\nabla \cdot G - G \nabla f$ (see Section~\ref{Sec:Langevin} for a review).
The difference between~\eqref{Eq:SDE0} and the ideal drift is $\tilde \mu = \mu - \nabla \cdot G + G \, \nabla f(X_t)$.
In Lemma~\ref{Lem:Bound} we show that 
$\frac{3}{4} I \, \preceq \, G \, \preceq \, \frac{4}{3} I$ and 
$\|\tilde \mu\| \le \frac{4}{3} t\Lg \|\nabla f(X_t)\| +  6 tn^{\frac{3}{2}} \Lh$.
Using these bounds and the LSI assumption, we can show the time derivative of relative entropy along~\eqref{Eq:SDE0} is bounded by
\begin{align}\label{Eq:DiffIneq0}
\frac{d}{dt} H_\nu(\rho_t) \le -\alpha H_\nu(\rho_t) + 16 t^2 n\left(\Lg^3 + 9 n^2 \Lh^2\right).
\end{align}
Integrating~\eqref{Eq:DiffIneq0} for $0 \le t \le \epsilon$ yields the desired bound~\eqref{Eq:OneStep}.
See Section~\ref{Sec:ProofLemOneStep} for a full proof.
\end{proof}

\subsection{Proof of Theorem~\ref{Thm:Main}}
\label{Sec:ThmMainProof}

\begin{proof}[Proof of Theorem~\ref{Thm:Main}]
By iterating the bound from Lemma~\ref{Lem:OneStep}, we have
\begin{align*}
H_\nu(\rho_k) 
&\le e^{-\alpha \epsilon k} H_\nu(\rho_0) + \frac{32 \, \epsilon^3 n (\Lg^3 + 9n^2 \Lh^2)}{1-e^{-\alpha \epsilon}} \\
&\le e^{-\alpha \epsilon k} H_\nu(\rho_0) + \frac{34 \, \epsilon^2 n (\Lg^3 + 9n^2 \Lh^2)}{\alpha}
\end{align*}
where in the last step we use $1-e^{-c} \ge \frac{16}{17} c$ for $0 < c = \alpha \epsilon \le \frac{3}{32}$, which holds because $\epsilon \le \frac{3}{32} \frac{\alpha}{L^2} \le \frac{3}{32 \alpha}$ by assumption.
\end{proof}

\section{Convergence in R\'enyi divergence}
\label{Sec:ResultRenyi}

Before stating our next result, we review the definition and some properties of R\'enyi divergence.

\subsection{R\'enyi divergence}

The {\em R\'enyi divergence} of order $q > 0$, $q \neq 1$, of a probability distribution $\rho$ with respect to $\nu$ is
\begin{align}\label{Eq:RenyiDef}
R_{q,\nu}(\rho) = \frac{1}{q-1} \log \int_{\R^n} \frac{\rho(x)^q}{\nu(x)^{q-1}} dx. 
\end{align}
As $q \to 1$, R\'enyi divergence recovers the relative entropy (KL divergence): $\lim_{q \to 1} R_{q,\nu}(\rho) = H_\nu(\rho)$.
R\'enyi divergence satisfies $R_{q,\nu}(\rho) \ge 0$ for all $\rho$, and $R_{q,\nu}(\rho) = 0$ if and only if $\rho = \nu$.
Furthermore, $q \mapsto R_{q,\nu}(\rho)$ is increasing. 
Therefore, R\'enyi divergence of order $q > 1$ is a family of stronger generalizations of KL divergence.
R\'enyi divergence has fundamental applications in statistics, physics, and computer science~\cite{R61,DR16,ACGM16,M17,C95,MPV00,VH14,BCG19}.
We recall R\'enyi divergence converges exponentially fast along the Langevin dynamics under LSI; see Section~\ref{Sec:Langevin}.

Convergence guarantee of R\'enyi divergence for sampling in discrete time was first studied in~\cite{VW19}, who show that R\'enyi divergence converges along ULA to its biased limit $\nu_\epsilon^{\text{ULA}}$ at the same rate as along the Langevin dynamics when $\nu_\epsilon^{\text{ULA}}$ itself satisfies either LSI or Poincar\'e inequality.
We will show a similar convergence guarantee for PLA in Section~\ref{Sec:PLARenyi}.
Thus, the iteration complexity is dominated by the bias $R_{q,\nu}(\nu_\epsilon)$.
We recall the bias of ULA can be infinite for large enough $q$, even in the Gaussian case~\cite[Example~3]{VW19}.
On the other hand, the bias of PLA in the Gaussian case is always finite and smaller than the bias of ULA, as we show in the following example.

\begin{example}\label{Ex:RenyiGaussian} 
Let $\nu = \N(0,\frac{1}{\alpha} I)$.
For $\epsilon > 0$, the limit of PLA is $\nu_\epsilon^{\text{PLA}} = \N(0,\frac{1}{\alpha(1+\frac{\epsilon\alpha}{2})} I)$,
and the bias is finite for all $q > 1$:
\begin{align*}
R_{q,\nu}(\nu_\epsilon^{\text{PLA}}) = \frac{n}{2(q-1)} \left(q\log\left(1+\frac{\epsilon\alpha}{2}\right)-\log\left(1+\frac{q\epsilon\alpha}{2}\right)\right).
\end{align*}
On the other hand, for $0 < \epsilon < \frac{2}{\alpha}$, the limit of ULA is $\nu_\epsilon^{\text{ULA}} = \N(0,\frac{1}{\alpha(1-\frac{\epsilon\alpha}{2})} I)$, and the bias is:
\begin{align*}
R_{q,\nu}(\nu_\epsilon^{\text{ULA}}) = 
\begin{cases}
\frac{n}{2(q-1)} \left(q\log\left(1-\frac{\epsilon\alpha}{2}\right)-\log\left(1-\frac{q\epsilon\alpha}{2}\right)\right)
~~~& \text{ for } 1 < q < \frac{2}{\epsilon\alpha}, \\
\infty & \text{ for } q \ge  \frac{2}{\epsilon\alpha}.
\end{cases}
\end{align*}
For $1 < q < \frac{2}{\epsilon\alpha}$, we have
$R_{q,\nu}(\nu_\epsilon^{\text{PLA}})  < R_{q,\nu}(\nu_\epsilon^{\text{ULA}})$.
\end{example}

\subsection{Convergence of R\'enyi divergence along PLA under LSI}
\label{Sec:PLARenyi}

Our second main result is the following convergence guarantee in R\'enyi divergence along PLA, assuming the biased limit $\nu_\epsilon$ satisfies LSI.
We provide the proof of Theorem~\ref{Thm:RenyiRate} in Section~\ref{App:ProofThmRenyiRate}.

\begin{theorem}\label{Thm:RenyiRate}
Assume $\nu_\epsilon$ satisfies LSI with constant $\beta > 0$,
$\nu$ is $(L,\infty)$-smooth, and $0 < \epsilon < \min\left\{\frac{1}{L}, \frac{1}{2\beta}\right\}$.
Let $q > 1$.
Then along PLA, for all $k \ge 0$,
\begin{align}
R_{q,\nu}(\rho_k) \le \left(\frac{q-\frac{1}{2}}{q-1}\right) R_{2q,\nu_\epsilon}(\rho_0) e^{-\frac{\beta \epsilon k}{2q}} + R_{2q-1,\nu}(\nu_\epsilon).
\end{align}
\end{theorem}

This result shows the iteration complexity for R\'enyi divergence along PLA depends on the bias.
For $\delta > 0$, let $h_q(\delta) = \sup \{ \epsilon > 0 \colon R_{2q-1,\nu}(\nu_\epsilon) \le \delta \}$, and assume $\delta$ is small so $h_q(\delta) < \min\{\frac{1}{L},\frac{1}{2\beta}\}$.
Theorem~\ref{Thm:RenyiRate} states to achieve $R_{q,\nu}(\rho_k) \le 2\delta$, it suffices to run PLA with step size $\epsilon = \Theta(h_q(\delta))$ for 
\begin{align}\label{Eq:RenyiIterLSI}
k = O\left(\frac{1}{\beta \epsilon} \log \frac{R_{2q,\nu_\epsilon}(\rho_0)}{\delta}\right)
\end{align}
iterations.
If we choose $\rho_0$ to be a proximal step from a Gaussian, then the initial R\'enyi divergence scales with $n$, as we show below.
Here $x^\ast$ is a stationary point for $f$ ($\nabla f(x^\ast) = 0$).

\begin{lemma}\label{Lem:InitRenyi}
Assume $\nu$ is $(L,\infty)$-smooth, and $0 < \epsilon < \frac{1}{L}$.
Let $\rho_0 = (I + \epsilon \nabla f)^{-1}_\#  \N(x^\ast,2\epsilon I)$ (concretely, $x_0 \sim \rho_0$ solves $x_0 + \epsilon \nabla f(x_0) = \tilde x_0$ where $\tilde x_0 \sim \N(x^\ast,2\epsilon I)$).
For all $q \ge 1$, $R_{q,\nu_\epsilon}(\rho_0) \le \tilde O(n)$.
\end{lemma}

Thus, Theorem~\ref{Thm:RenyiRate} yields an iteration complexity of 
\begin{align}\label{Eq:RenyiCompLSI}
k = \tilde O\left(\frac{1}{\beta h_q(\delta)}\right)
\end{align}
for PLA under LSI to reach $R_{q,\nu}(\rho_k) \le 2\delta$ with $\epsilon = \Theta(h_q(\delta))$.
For example, if $h_q(\delta) = \Omega(\delta)$, then the iteration complexity is $k = \tilde O(\frac{1}{\beta \delta})$. 
If $h_q(\delta) = \Omega(\sqrt{\delta})$, as in the Gaussian case (Example~\ref{Ex:RenyiGaussian}), then the iteration complexity is $k = \tilde O(\frac{1}{\beta \delta^{1/2}})$. 
However, in general we do not know how to control this bias.

\subsection{Convergence of R\'enyi divergence along PLA under Poincar\'e}
\label{Sec:PLARenyiPoincare}

We recall $\nu$ satisfies {\em Poincar\'e inequality} with a constant $\alpha > 0$ if for all smooth $g \colon \R^n \to \R$,
\begin{align*} 
\Var_\nu(g) \le \frac{1}{\alpha} \E_\nu[\|\nabla g\|^2]
\end{align*}
where $\Var_\nu(g) = \E_\nu[g^2] - \E_\nu[g]^2$ is the variance of $g$ under $\nu$.
Poincar\'e inequality is an isoperimetry condition which is weaker than LSI.
LSI implies Poincar\'e inequality with the same constant, and in fact Poincar\'e inequality is a linearization of LSI~\cite{R81,V03}.
Like LSI, Poincar\'e inequality is preserved under bounded perturbation and Lipschitz mapping.
However, Poincar\'e inequality is more general than LSI; for example, distributions satisfying LSI have sub-Gaussian tails, while distributions satisfying Poincar\'e inequality can have sub-exponential tails.
Whereas LSI is equivalent to a bound on the log-Cheeger constant, Poincar\'e inequality is equivalent to a bound on the Cheeger constant~\cite{L94}.
We recall when $\nu$ satisfies Poincar\'e inequality, R\'enyi divergence converges along the Langevin dynamics at a rate which is initially linear then exponential; see Section~\ref{Sec:Langevin} for a review.

Our third main result is the following convergence guarantee in R\'enyi divergence along PLA, assuming the biased limit $\nu_\epsilon$ satisfies Poincar\'e inequality.
We provide the proof of Theorem~\ref{Thm:RenyiRatePoincare} in Section~\ref{App:RenyiRatePoincare}.

\begin{theorem}\label{Thm:RenyiRatePoincare}
Assume $\nu_\epsilon$ satisfies Poincar\'e inequality with constant $\beta > 0$,
$\nu$ is $(L,\infty)$-smooth, and $0 < \epsilon < \min\left\{\frac{1}{L}, \frac{1}{2\beta}\right\}$.
Let $q > 1$.
Then along PLA, for $k \ge k_0 := \frac{2q}{\beta\epsilon}(R_{2q,\nu_\epsilon}(\rho_0)-1)$,
\begin{align}
R_{q,\nu}(\rho_k) \le \left(\frac{q-\frac{1}{2}}{q-1}\right) e^{-\frac{\beta \epsilon (k-k_0)}{2q}} + R_{2q-1,\nu_\epsilon}(\nu).
\end{align}
\end{theorem}

This result shows the iteration complexity for R\'enyi divergence along PLA depends on the bias.
For $\delta > 0$, let $h_q(\delta) = \sup \{ \epsilon > 0 \colon R_{2q-1,\nu}(\nu_\epsilon) \le \delta \}$, and assume $\delta$ is small so $h_q(\delta) < \min\{\frac{1}{L},\frac{1}{2\beta}\}$.
Theorem~\ref{Thm:RenyiRatePoincare} states to achieve $R_{q,\nu}(\rho_k) \le 2\delta$, it suffices to run PLA with step size $\epsilon = \Theta(h_q(\delta))$ for 
\begin{align}\label{Eq:RenyiIterPI}
k = O\left(\frac{1}{\beta \epsilon} \left(R_{2q,\nu_\epsilon}(\rho_0) + \log \frac{1}{\delta}\right) \right)
\end{align}
iterations.
Note the dependence on $R_{2q,\nu_\epsilon}(\rho_0)$ is now linear, rather than logarithmic under LSI~\eqref{Eq:RenyiIterLSI}.
As in Lemma~\ref{Lem:InitRenyi}, if we choose $\rho_0$ to be a proximal step from a Gaussian, then $R_{2q,\nu_\epsilon}(\rho_0) \le \tilde O(n)$.
Thus, Theorem~\ref{Thm:RenyiRatePoincare} yields an iteration complexity of 
\begin{align}
k = \tilde O\left(\frac{n}{\beta h_q(\delta)}\right)
\end{align}
for PLA under Poincar\'e to reach $R_{q,\nu}(\rho_k) \le 2\delta$ with $\epsilon = \Theta(h_q(\delta))$.
This is a factor of $n$ larger than the complexity under LSI~\eqref{Eq:RenyiCompLSI}.

\section{A review on Langevin dynamics}
\label{Sec:Langevin}

\paragraph{Notation.}
For a matrix $A \in \R^{n \times n}$, let $\|A\| \equiv \|A\|_\op$ denote the operator norm and $\|A\|_{\HS}$ the Hilbert-Schmidt norm. 
If $A$ is symmetric with eigenvalues $\lambda_1,\dots,\lambda_n \in \R$, 
then $\|A\|_\op = \max_i |\lambda_i|$ and $\|A\|_{\HS} = (\sum_{i=1}^n \lambda_i^2)^{1/2}$.
Note that $\|A\|_{\HS} \le \sqrt{n} \|A\|_\op$.

For a differentiable function $\phi \colon \R^n \to \R$, let $\nabla \phi(x) = \left(\part{\phi(x)}{x_1}, \dots, \part{\phi(x)}{x_n}\right) \in \R^n$ denote the gradient vector, $\nabla^2 \phi(x) = \left(\part{^2 \phi(x)}{x_i \, \partial x_j}\right)_{ij} \in \R^{n \times n}$ the Hessian matrix, and $\nabla^3 \phi(x) = \left(\part{^3 \phi(x)}{x_i \, \partial x_j \, \partial x_k}\right)_{ijk} \in \R^{n \times n \times n}$ the tensor of third-order derivatives.
Let $\Delta \phi(x) = \Tr(\nabla^2 \phi(x)) = \sum_{i=1}^n \part{^2 \phi(x)}{x_i^2} \in \R$ denote the Laplacian of $\phi$.

Let $\nabla \cdot$ denote the divergence operator that acts on a vector field $v(x) = (v_1(x),\dots,v_n(x)) \in \R^n$ by $\nabla \cdot v(x) = \sum_{i=1}^n \part{v_i(x)}{x_i} \in \R$.
The divergence of gradient is the Laplacian: $\nabla \cdot (\nabla \phi) = \Delta \phi$.
We will use the integration by parts formula: $\int_{\R^n} \langle \nabla \phi(x), v(x) \rangle \, dx = -\int_{\R^n} \phi(x) \nabla \cdot v(x) \, dx$, where the boundary term is zero if $\phi, v$ have sufficiently fast decay at infinity.

For a matrix-valued function $G \colon \R^n \to \R^{n \times n}$, let $\Tr(\nabla^3 \phi(x) \, G(x)) \in \R^n$ denote the vector whose $i$-th component is $\Tr(\nabla_i \nabla^2 \phi(x) \, G(x))$ where $\nabla_i \nabla^2 \phi(x) = \left(\part{^3 \phi(x)}{x_i \, \partial x_j \, \partial x_k}\right)_{jk} \in \R^{n \times n}$.
Let $\nabla \cdot G(x) \in \R^n$ denote the vector whose $i$-th component is $\nabla \cdot G_i(x) \in \R$, where $G_i(x) \in \R^n$ is the $i$-th row of $G(x)$.
Let $\langle \nabla^2, G(x) \rangle = \nabla \cdot (\nabla \cdot G(x)) = \sum_{i,j=1}^n \part{^2 G_{ij}(x)}{x_i \, \partial x_j} \in \R$.

\subsection{Weighted Langevin dynamics}

Let $G \colon \R^n \to \R^{n \times n}$ be a differentiable matrix-valued function where $G(x) \succ 0$ is positive definite.
Recall the {\em weighted Langevin dynamics} for $\nu = e^{-f}$ with covariance $G$ is the SDE
\begin{align}\label{Eq:WLD}
dX_t = (\nabla \cdot G(X_t) - G(X_t) \nabla f(X_t)) \, dt + \sqrt{2G(X_t)} \, dW_t
\end{align}
where $(W_t)_{t \ge 0}$ is the standard Brownian motion on $\R^n$.
The drift term above is chosen to ensure $\nu = e^{-f}$ is a stationary measure for the weighted Langevin dynamics~\eqref{Eq:WLD}.
This is apparent from the following Fokker-Planck equation;
see also~\cite{MCF15,LWME19}.

\begin{lemma}\label{Lem:WLDFP}
If $X_t$ evolves following the weighted Langevin dynamics~\eqref{Eq:WLD}, then the density $\rho_t$ evolves following
\begin{align}\label{Eq:WLDFP}
\part{\rho_t}{t} = \nabla \cdot \left(\rho_t G \nabla \log \frac{\rho_t}{\nu}\right).
\end{align}
\end{lemma}
\begin{proof}
Recall for a general Langevin dynamics $dX_t = b(X_t) \, dt + \sqrt{2G(X_t)} \, dW_t$, the Fokker-Planck equation for the density $\rho_t$ of $X_t$ is (see for example~\cite{Mac92,WJL17}):
\begin{align}\label{Eq:FPWant}
\part{\rho}{t} = -\nabla \cdot (\rho b) + \langle \nabla^2, \rho G \rangle
\end{align}
where for simplicity we write $\rho$ in place of $\rho_t$.
For the drift $b = \nabla \cdot G - G \nabla f$ in~\eqref{Eq:WLDFP}, we have
\begin{align*}
\nabla \cdot \left(\rho G \nabla \log \frac{\rho}{\nu}\right)
&= \nabla \cdot (G \nabla \rho + \rho G \nabla f) \\
&= \nabla \cdot ( \nabla \cdot (\rho G) - \rho \nabla \cdot G + \rho G \nabla f) \\
&= \nabla \cdot ( \nabla \cdot (\rho G) - \rho b) \\
&= \langle \nabla^2, \rho G \rangle - \nabla \cdot( \rho b)
\end{align*}
which matches~\eqref{Eq:FPWant}, as desired.
\end{proof}

From~\eqref{Eq:WLDFP} it is clear that $\nu$ is a stationary measure for the weighted Langevin dynamics~\eqref{Eq:WLD}.
Furthermore, we can quantify how much the KL divergence with respect to $\nu$ decreases along~\eqref{Eq:WLD}.

We define the {\em (weighted) relative Fisher information} of $\rho$ with respect to $\nu$ to be
\begin{align}
J_{\nu,G}(\rho) = \int_{\R^n} \rho(x) \, \left\| \nabla \log \frac{\rho(x)}{\nu(x)} \right\|^2_{G(x)} \, dx.
\end{align}
Here $\|v\|^2_G := \langle v, Gv \rangle$ is the weighted norm of a vector $v \in \R^n$ by a positive definite matrix $G \succ 0$.
Then we have the following generalization of De Bruijn's identity.

\begin{lemma}
Along the weighted Langevin dynamics~\eqref{Eq:WLD},
\begin{align}\label{Eq:DeBruijn}
\frac{d}{dt} H_\nu(\rho_t) = -J_{\nu,G}(\rho_t).
\end{align}
\end{lemma}
\begin{proof}
Using the Fokker-Planck equation~\eqref{Eq:WLDFP} and integration by parts,
\begin{align*}
\frac{d}{dt} H_\nu(\rho_t) &= \int_{\R^n} \part{\rho_t}{t} \, \log \frac{\rho_t}{\nu} \, dx \\
&= \int_{\R^n}  \nabla \cdot \left(\rho_t G \nabla \log \frac{\rho_t}{\nu}\right) \, \log \frac{\rho_t}{\nu} \, dx  \\
&= - \int_{\R^n}  \rho_t \left \langle G \nabla \log \frac{\rho_t}{\nu}, \nabla \log \frac{\rho_t}{\nu} \right\rangle\, dx  \\
&= -J_{\nu,G}(\rho_t).
\end{align*}
\end{proof}

\subsection{Unweighted Langevin dynamics}

The {\em (unweighted) Langevin dynamics} is when the covariance is the identity matrix ($G(x) = I$):
\begin{align}\label{Eq:ULD}
dX_t = - \nabla f(X_t) \, dt + \sqrt{2} \, dW_t
\end{align}
In this case the unweighted relative Fisher information is the usual one from~\eqref{Eq:J}: $J_{\nu,I}(\rho) = J_\nu(\rho)$.
Then~\eqref{Eq:DeBruijn} becomes the usual De Bruijn's identity: $\frac{d}{dt} H_\nu(\rho_t) = -J_\nu(\rho_t)$.
We see that under LSI~\eqref{Eq:LSI} we have $\frac{d}{dt} H_\nu(\rho_t) \le -2\alpha H_\nu(\rho_t)$, which implies KL divergence converges exponentially fast:
\begin{align*}
H_\nu(\rho_t) \le e^{-2\alpha t} H_\nu(\rho_0).
\end{align*}
We recall the interpretation of the Langevin dynamics~\eqref{Eq:ULD} as the gradient flow of KL divergence in the space of measures with the Wasserstein metric, with LSI as the gradient domination condition~\cite{JKO98,OV00}.

Under LSI, we can also show R\'enyi divergence of order $q \ge 1$ converges exponentially fast along the Langevin dynamics:
\begin{align*}
R_{q,\nu}(\rho_t) \le e^{-\frac{2\alpha}{q} t} R_{q,\nu}(\rho_0),
\end{align*}
see for example~\cite[Theorem~3]{VW19}.
We also recall the interpretation of the Langevin dynamics~\eqref{Eq:ULD} as the gradient flow of R\'enyi divergence in the space of measures with a suitably defined metric (which depends on $\nu$), with LSI as the gradient domination condition~\cite{CLL18}.

Under Poincar\'e inequality, we can show R\'enyi divergence of order $q \ge 2$ converges at a rate which is initially linear then exponential:
\begin{align*}
R_{q,\nu}(\rho_t) \le
\begin{cases}
R_{q,\nu}(\rho_0) -\frac{2\alpha t}{q} ~~ & \text{ if } R_{q,\nu}(\rho_0) \ge 1 \text{ and as long as } R_{q,\nu}(\rho_t) \ge 1, \\
e^{-\frac{2\alpha t}{q}} R_{q,\nu}(\rho_0) ~~ & \text{ if } R_{q,\nu}(\rho_0) \le 1,
\end{cases}
\end{align*}
see for example~\cite[Theorem~5]{VW19}.

\section{Proofs for Section~\ref{Sec:PLA}}
\label{Sec:ProofPLA}

\subsection{Proof of Lemma~\ref{Lem:OneStep}}
\label{Sec:ProofLemOneStep}

\begin{proof}[Proof of Lemma~\ref{Lem:OneStep}]
The output $x_{k+1}$ of PLA~\eqref{Eq:PLA} is the value at time $t = \epsilon$ of the stochastic process
\begin{align}\label{Eq:Stoch1q}
X_t = X_0 - t \nabla f(X_t) + \sqrt{2} W_t
\end{align}
starting at $X_0 = x_k$, where $W_t$ is the standard Brownian motion in $\R^n$.
By Lemma~\ref{Lem:SDERep}, $(X_t)_{t \ge 0}$~\eqref{Eq:Stoch1q} evolves following the SDE
\begin{align}\label{Eq:SDEq}
dX_t = \mu \, dt + \sqrt{2 G} \, dW_t
\end{align}
where $\mu = -\sqrt{G} \left( \nabla f(X_t) + t \, \Tr(\nabla^3 f(X_t) G)\right)$ and $G = (I + t \nabla^2 f(X_t))^{-2}$.
Recall the Langevin dynamics with covariance $G$ converges to $\nu = e^{-f}$ if the drift is $\nabla \cdot G - G \nabla f$ (see Section~\ref{Sec:Langevin}).
We write the SDE~\eqref{Eq:SDEq} as
\begin{align}\label{Eq:SDEnew}
dX_t = (\nabla \cdot G - G \nabla f(X_t) + \tilde \mu) \, dt + \sqrt{2G} \, dW_t
\end{align}
where $\tilde \mu$ is the shifted drift:
\begin{align*}
\tilde \mu &= \mu - \nabla \cdot G + G \, \nabla f(X_t) \\
&=  - t \nabla^2 f(X_t) G \, \nabla f(X_t) - t \sqrt{G} \, \Tr(\nabla^3 f(X_t) \, G) - \nabla \cdot G. 
\end{align*}
The Fokker-Planck equation of the SDE~\eqref{Eq:SDEnew} for one step of PLA is then
\begin{align*}
\part{\rho}{t} = \nabla \cdot \left( \rho G \nabla \log \frac{\rho}{\nu} \right) - \nabla \cdot (\rho \tilde \mu).
\end{align*}
The time derivative of KL divergence is, by integration by parts,
\begin{align*}
\frac{d}{dt} H_\nu(\rho) &= \int_{\R^n} \part{\rho}{t} \log \frac{\rho}{\nu} \, dx \\
&= \int_{\R^n}  \nabla \cdot \left( \rho G \nabla \log \frac{\rho}{\nu} \right) \log \frac{\rho}{\nu} \, dx - \int_{\R^n} \nabla \cdot (\rho \tilde \mu) \log \frac{\rho}{\nu} \, dx \\
&= -\E_\rho\left[\left\| \nabla \log \frac{\rho}{\nu} \right\|^2_G \right] + \E_\rho\left[ \left \langle \tilde \mu, \nabla \log \frac{\rho}{\nu}  \right \rangle \right].
\end{align*}
Since $G \succeq \frac{3}{4} I$ by Lemma~\ref{Lem:Bound}, and using $\langle a,b \rangle \le 2\|a\|^2 + \frac{1}{8} \|b\|^2$, we have 
\begin{align*}
\frac{d}{dt} H_\nu(\rho) 
&\le -\frac{3}{4} \E_\rho\left[\left\| \nabla \log \frac{\rho}{\nu} \right\|^2 \right] + 2 \E_\rho[\| \tilde \mu \|^2] + \frac{1}{8} \E_\rho\left[\left\| \nabla \log \frac{\rho}{\nu} \right\|^2 \right] \\
&= -\frac{5}{8} J_\nu(\rho) + 2 \E_\rho[\| \tilde \mu \|^2].
\end{align*}
Then by LSI $J_\nu(\rho) \ge 2\alpha H_\nu(\rho)$,
\begin{align}\label{Eq:Calc3}
\frac{d}{dt} H_\nu(\rho) 
\le -\frac{5\alpha}{4} H_\nu(\rho) + 2 \E_\rho[\| \tilde \mu \|^2].
\end{align}
By the bound~\eqref{Eq:Want2} in Lemma~\ref{Lem:Bound} and using $(a+b)^2 \le 2a^2 + 2b^2$, we have
\begin{align*}
\E_\rho[\| \tilde \mu \|^2] &\le \E_\rho\left[ 2\left(\frac{4}{3} t\Lg \|\nabla f\| \right)^2 +  2(6 tn^{\frac{3}{2}} \Lh)^2 \right] \\
&= \frac{32}{9} t^2 \Lg^2 \E_\rho[\|\nabla f\|^2] + 72 t^2 n^3 \Lh^2.
\end{align*}
We recall LSI implies Talagrand's inequality, which implies the following bound (see~\cite[Lemma~12]{VW19}):
\begin{align*}
\E_\rho[\|\nabla f\|^2]  \le \frac{4\Lg^2}{\alpha} H_\nu(\rho) + 2n\Lg.
\end{align*}
Then
\begin{align*}
\E_\rho[\| \tilde \mu \|^2]
&\le \frac{32}{9} t^2 \Lg^2 \left(\frac{4\Lg^2}{\alpha} H_\nu(\rho) +  2n\Lg \right) + 72 t^2 n^3 \Lh^2 \\
&= \frac{128t^2 L^4}{9\alpha} H_\nu(\rho) + \frac{64}{9} t^2 n \Lg^3 + 72 t^2 n^3 \Lh^2.
\end{align*}
Plugging this to~\eqref{Eq:Calc3}, we obtain
\begin{align}
\frac{d}{dt} H_\nu(\rho) 
&\le \left(-\frac{5\alpha}{4} + \frac{256 t^2 L^4}{9\alpha}\right)  H_\nu(\rho) + 16 t^2 n\left(\frac{8}{9} \Lg^3 + 9 n^2 \Lh^2\right) \notag \\
&\le -\alpha H_\nu(\rho) + 16 t^2 n\left(\Lg^3 + 9 n^2 \Lh^2\right) \label{Eq:DiffIneq}
\end{align}
where the last inequality above holds for $0 \le t \le \frac{3}{32} \frac{\alpha}{\Lg^2}$.

We wish to integrate the differential inequality~\eqref{Eq:DiffIneq} for $0 \le t \le \epsilon$.
First using $t \le \epsilon$, we have
\begin{align*}
\frac{d}{dt} H_\nu(\rho_t) 
&\le -\alpha H_\nu(\rho_t) + \epsilon^2 C
\end{align*}
where $C = 16 n(\Lg^3 + 9 n^2 \Lh^2)$.
Multiplying both sides by $e^{\alpha t}$, we can write the above as
\begin{align*}
\frac{d}{dt} \left(e^{\alpha t} H_\nu(\rho_t)\right) \le e^{\alpha t} \epsilon^2 C.
\end{align*}
Integrating from $t=0$ to $t=\epsilon$ gives
\begin{align*}
e^{\alpha \epsilon} H_\nu(\rho_\epsilon) - H_\nu(\rho_0) 
\;\le\; \left(\frac{e^{\alpha \epsilon} - 1}{\alpha}\right) \epsilon^2 C
\;\le\; 2\epsilon^3 C
\end{align*}
where in the last step we use $e^c \le 1+2c$ for $0 < c = \alpha \epsilon \le 1$, which holds because $0 < \epsilon \le \frac{3}{32} \frac{\alpha}{\Lg^2} < \frac{1}{\alpha}$.
Therefore, we obtain the bound
\begin{align*}
H_\nu(\rho_\epsilon) &\le e^{-\alpha \epsilon} H_\nu(\rho_0) + 2e^{-\alpha \epsilon} \epsilon^3 C \\
&\le e^{-\alpha \epsilon} H_\nu(\rho_0) + 2 \epsilon^3 C \\
&= e^{-\alpha \epsilon} H_\nu(\rho_0) + 32 \epsilon^3 n(\Lg^3 + 9 n^2 \Lh^2)
\end{align*}
as desired.
\end{proof}

\subsection{SDE representation of one step of PLA}

The output $x_{k+1}$ of PLA~\eqref{Eq:PLA} is the value at time $t = \epsilon$ of the stochastic process $(X_t)_{t \ge 0}$ given by
\begin{align}\label{Eq:Stoch1}
X_t = X_0 - t \nabla f(X_t) + \sqrt{2} W_t
\end{align}
starting at $X_0 = x_k$, where $W_t$ is the standard Brownian motion in $\R^n$.

\begin{lemma}\label{Lem:SDERep}
The stochastic process $(X_t)_{t \ge 0}$~\eqref{Eq:Stoch1} evolves following the SDE
\begin{align}\label{Eq:SDE}
dX_t = \mu(X_t,t) \, dt + \sqrt{2 \, G(X_t,t)} \, dW_t
\end{align}
where
\begin{align}
\mu(x,t) &= -(I + t \nabla^2 f(x))^{-1} \left( \nabla f(x) + t \, \Tr(\nabla^3 f(x) (I + t \nabla^2 f(x))^{-2})\right) \label{Eq:MuG1} \\
G(x,t) &= (I + t \nabla^2 f(x))^{-2}.  \label{Eq:MuG2}
\end{align}
\end{lemma}
\begin{proof}
For $t \ge 0$, let
\begin{align*}
\tilde X_t = X_t + t \nabla f(X_t)
\end{align*}
so $\tilde X_0 = X_0$, and we can write~\eqref{Eq:Stoch1} as
\begin{align*}
\tilde X_t = \tilde X_0 + \sqrt{2} W_t.
\end{align*}
That is, $(\tilde X_t)_{t \ge 0}$ evolves following the SDE
\begin{align}\label{Eq:SDE1}
d\tilde X_t = \sqrt{2} dW_t.
\end{align}
Suppose $(X_t)_{t \ge 0}$ evolves by
\begin{align*}
dX_t = \mu \, dt + \sqrt{2 \, G} \, dW_t
\end{align*}
for some $\mu \equiv \mu(X_t,t)$ and $G \equiv G(X_t,t) \succ 0$.
Let $T_t(x) = x + t \nabla f(x)$, so $\part{T_t}{t}(x) = \nabla f(x)$, $\nabla T_t(x) = I + t \nabla^2 f(x)$, and $\nabla^2 T_t(x) = t \nabla^3 f(x)$.
Then by It\^o's lemma for $\tilde X_t = T_t(X_t)$, we have
\begin{align}
d\tilde X_t &= dT_t(X_t) = \left(\part{T_t}{t}(X_t) + \nabla T_t(X_t)^\top \mu + \Tr\left(\nabla^2 T_t(X_t) \, G \right) \right) \, dt + \sqrt{2} \, \nabla T_t(X_t) \, \sqrt{G} \, dW_t  \notag \\
&=  \left(\nabla f(X_t) + (I + t \nabla^2 f(X_t)) \, \mu + t \, \Tr\left(\nabla^3 f(X_t) \, G \right) \right) \, dt + \sqrt{2} \, (I + t \nabla^2 f(X_t)) \, \sqrt{G} \, dW_t
\label{Eq:SDE2}
\end{align}
Matching~\eqref{Eq:SDE1} and~\eqref{Eq:SDE2}, we must have
\begin{align*}
\nabla f(x) + (I + t \nabla^2 f(x)) \, \mu + t \, \Tr\left(\nabla^3 f(x) \, G \right) &= 0 \\
(I + t \nabla^2 f(x)) \, \sqrt{G}  &= I
\end{align*}
which implies
\begin{align*}
\mu &= - (I + t \nabla^2 f(x))^{-1} \left(\nabla f(x) + t \, \Tr\left(\nabla^3 f(x) \, (I + t \nabla^2 f(x))^{-2} \right) \right)  \\
G &= (I + t \nabla^2 f(x))^{-2}
\end{align*}
as desired. 
\end{proof}

\subsection{Bounds under smoothness}

Recall $\mu, G$ are defined in~\eqref{Eq:MuG1},~\eqref{Eq:MuG2}, and $\tilde \mu$ is the shifted drift:
\begin{align}
\tilde \mu &= \mu - \nabla \cdot G + G \, \nabla f(x) \notag  \\
&=  - t \nabla^2 f(x) G \, \nabla f(x)
- t \sqrt{G} \, \Tr(\nabla^3 f(x) \, G) - \nabla \cdot G.  \label{Eq:mu}
\end{align}

\begin{lemma}\label{Lem:Bound}
Assume $\nu$ is $(\Lg,\Lh)$-smooth.
For $0 \le t \le \min\{ \frac{1}{8\Lg}, \frac{1}{\Lh} \}$, we have the following bounds:
\begin{align}
\frac{3}{4} I \, \preceq \, G \, &\preceq \, \frac{4}{3} I   \label{Eq:Want1} \\
\|\tilde \mu\| &\le \frac{4}{3} t\Lg \|\nabla f(x)\| +  6 tn^{\frac{3}{2}} \Lh.  \label{Eq:Want2}
\end{align}
\end{lemma}
\begin{proof}
Since $-\Lg I \preceq \nabla^2 f(x) \preceq \Lg I$, we have 
\begin{align*}
\frac{1}{(1+t\Lg)^2} I \; \preceq \;  G = (I + t \nabla^2 f(x))^{-2} \; \preceq \;  \frac{1}{(1-t\Lg)^2} I.
\end{align*}
For $0 \le t \le \frac{1}{8\Lg}$, we have
\begin{align*}
\frac{3}{4} I \prec \left(\frac{8}{9}\right)^2 I \; \preceq \;  G \; \preceq \; \left(\frac{8}{7}\right)^2 I \prec \frac{4}{3} I
\end{align*}
as claimed in~\eqref{Eq:Want1}.

To bound $\tilde \mu$~\eqref{Eq:mu}, we apply triangle inequality:
\begin{align}\label{Eq:ToShow}
\|\tilde \mu\| \le  \| t \nabla^2 f(x) G \nabla f(x)\| 
+ \|t \sqrt{G} \, \Tr(\nabla^3 f(x) \, G)\| 
+ \| \nabla \cdot G\|.
\end{align}
We now bound the three terms above:
\begin{enumerate}
  \item For the first term, we have
  \begin{align*}
  \|t \nabla^2 f(x) G \nabla f(x)\| &\le t \|\nabla^2 f(x)\|_\op \|G\|_\op \|\nabla f(x)\| \\
  &\le \frac{4}{3} t\Lg \|\nabla f(x)\|.
  \end{align*}
  \item For the second term, we have:
  \begin{align*}
  \|t \sqrt{G} \, \Tr(\nabla^3 f(x) \, G)\| \le t \|\sqrt{G}\|_\op \|\Tr(\nabla^3 f(x) \, G)\|
  \end{align*}
  We have $\|\sqrt{G}\|_\op = \sqrt{\|G\|_\op} \le (\frac{4}{3})^{\frac{1}{2}}$.
  The $i$-th entry of $\Tr(\nabla^3 f(x) \, G)$ is $\Tr(\nabla_i \nabla^2 f(x) \, G)$, which is bounded:
  \begin{align*}
  |\Tr(\nabla_i \nabla^2 f(x) \, G)| &\le \|\nabla_i \nabla^2 f(x)\|_\HS \|G\|_\HS \\
  &\le n \|\nabla_i \nabla^2 f(x)\|_\op \|G\|_\op \\
  &\le \frac{4}{3} n\Lh.
  \end{align*}
  Therefore, $\|\Tr(\nabla^3 f(x) \, G)\| \le \frac{4}{3} n^{\frac{3}{2}} \Lh$.
  Then we can bound the second term of~\eqref{Eq:ToShow} by
  \begin{align*}
    \|t \sqrt{G} \, \Tr(\nabla^3 f(x) \, G)\| \le t \left(\frac{4}{3}\right)^{\frac{3}{2}} n^{\frac{3}{2}} \Lh < 2 tn^{\frac{3}{2}} \Lh.
  \end{align*}
  \item For the third term of~\eqref{Eq:ToShow}, we have 
  \begin{align*}
  \|\nabla \cdot G\|^2 = \sum_{i=1}^n (\nabla \cdot G_i)^2 
  &= \sum_{i=1}^n \Big( \sum_{j=1}^n \part{G_{ij}}{x_j} \Big)^2 \\
  &\le  n \sum_{i=1}^n \sum_{j=1}^n \Big(\part{G_{ij}}{x_j} \Big)^2 \\
  &\le  n \sum_{i=1}^n \sum_{j=1}^n \sum_{k=1}^n \Big(\part{G_{ij}}{x_k} \Big)^2 \\
  &= n \sum_{k=1}^n \|\nabla_k G\|^2_\HS \\
  &\le n^2 \sum_{k=1}^n \|\nabla_k G\|^2_\op.
  \end{align*}
  We now claim that for each $k = 1,\dots,n$,
  \begin{align}\label{Eq:ToShow2}
  \|\nabla_k G\|_\op \le 4t\Lh
  \end{align}
  which will imply the desired bound $\|\nabla \cdot G\| \le 4t n^{\frac{3}{2}} \Lh$. 

  To show~\eqref{Eq:ToShow2}, we will show that for each $x \in \R^n$, unit vector $v \in \R^n$, $\|v\|=1$, and $0 < u < \frac{1}{2}$,
  \begin{align}\label{Eq:ToShow3}
  \|G(x+uv) - G(x)\|_\op \le 4t\Lh(u + o(u)).
  \end{align}
  Since $\nabla^2 f$ is $\Lh$-Lipschitz, we have
  \begin{align*}
  \nabla^2 f(x) -u \Lh I \;\preceq\; \nabla^2 f(x+uv) \;\preceq\; \nabla^2 f(x) + u \Lh I.
  \end{align*}
  Therefore,
  \begin{align*}
  (1-ut\Lh)I + t \nabla^2 f(x) \;\preceq\; I + t \nabla^2 f(x+uv) \;\preceq\; (1+ut\Lh)I + t \nabla^2 f(x).
  \end{align*}
  Then
  \begin{align*}
  ((1+ut\Lh)I + t \nabla^2 f(x))^{-2} \;\preceq\; (I + t \nabla^2 f(x+uv))^{-2} \;\preceq\; ((1-ut\Lh)I + t \nabla^2 f(x))^{-2}.
  \end{align*}
  This implies
  \begin{align}
  &((1+ut\Lh)I + t \nabla^2 f(x))^{-2} - (I + t \nabla^2 f(x))^{-2} \label{Eq:LHS} \\
  & ~~~~~~ \preceq \; G(x+uv) - G(x) \;= \; (I + t \nabla^2 f(x+uv))^{-2} - (I + t \nabla^2 f(x))^{-2} \\
  & ~~~~~~ \preceq \; ((1-ut\Lh)I + t \nabla^2 f(x))^{-2} - (I + t \nabla^2 f(x))^{-2}.  \label{Eq:RHS}
  \end{align}
  For each eigenvalue $-L \le \lambda \le L$ of $\nabla^2 f(x)$, the eigenvalue of the left-hand side~\eqref{Eq:LHS} above is at most (in magnitude)
  \begin{align}
  \left|\frac{1}{(1+ut\Lh + t \lambda)^2} -   \frac{1}{(1 + t \lambda)^2} \right|
   &=  \frac{(1+ut\Lh + t \lambda)^2 - (1 + t \lambda)^2}{(1 + t \lambda)^2(1+ut\Lh + t \lambda)^2}  \notag \\
   &= \frac{ut\Lh(2 + ut\Lh + 2t \lambda)}{(1 + t \lambda)^2(1+ut\Lh + t \lambda)^2}  \notag \\
   &\le \frac{ut\Lh(2 + ut\Lh + 2t L)}{(1 - tL)^2(1+ut\Lh - tL)^2}  \notag \\
   &\le \frac{ut\Lh(2 + u + \frac{1}{4})}{(1 - \frac{1}{8})^2(1+0 - \frac{1}{8})^2}  \notag \\
   &= \frac{9}{4} \left(\frac{8}{7}\right)^4 ut\Lh \left(1+\frac{4}{9} u \right) \notag \\
   &< 4ut\Lh \left(1+\frac{4}{9} u\right).  \label{Eq:BdLHS}
  \end{align}
  Similarly, eigenvalue of the right-hand side~\eqref{Eq:RHS} above is at most
  \begin{align}
  \frac{1}{(1-ut\Lh + t \lambda)^2} -   \frac{1}{(1 + t \lambda)^2} 
   &=  \frac{(1 + t \lambda)^2 - (1-ut\Lh + t \lambda)^2}{(1 + t \lambda)^2(1-ut\Lh + t \lambda)^2}  \notag \\
   &= \frac{ut\Lh(2 - ut\Lh + 2t \lambda)}{(1 + t \lambda)^2(1-ut\Lh + t \lambda)^2}  \notag \\
   &\le \frac{ut\Lh(2 - ut\Lh + 2t L)}{(1 - t L)^2(1-ut\Lh - t L)^2}  \notag \\
   &\le \frac{ut\Lh(2 - 0 + \frac{1}{4})}{(1 - \frac{1}{8})^2(1-u - \frac{1}{8})^2}  \notag \\
   &= \frac{9}{4} \left(\frac{8}{7}\right)^4 \frac{ut\Lh}{(1-\frac{8}{7} u)^2} \notag \\
   &< \frac{4ut\Lh}{(1-\frac{8}{7} u)^2}.   \label{Eq:BdRHS}
  \end{align}
  Combining~\eqref{Eq:BdLHS} and~\eqref{Eq:BdRHS} gives
  \begin{align*}
  \|G(x+uv) - G(x)\|_\op &\le 4ut\Lh \max\left\{\left(1+\frac{4}{9} u\right), \frac{1}{(1-\frac{8}{7} u)^2}\right\}.
  \end{align*}
  Therefore, the partial derivative along direction $v$ is
  \begin{align*}
  \|\nabla_v G(x)\|_\op &= \left\| \lim_{u \to 0}  \frac{G(x+uv)-G(x)}{u} \right\|_\op \\
  &= \lim_{u \to 0} \frac{\|G(x+uv)-G(x)\|_\op}{u} \\
  &\le \lim_{u \to 0}  \; 4t\Lh \max\left\{\left(1+\frac{4}{9} u\right), \frac{1}{(1-\frac{8}{7} u)^2}\right\} \\
  &= 4t\Lh.
  \end{align*}
  In particular, when $v = e_k$, we have $\|\nabla_k G(x)\|_\op \le 4t\Lh$, as desired.  
\end{enumerate}
Plugging in the three bounds above to~\eqref{Eq:ToShow} yields the desired bound~\eqref{Eq:Want2} for $\tilde \mu$.
\end{proof}

\section{Proofs for Section~\ref{Sec:ResultRenyi}}
\label{Sec:ProofRenyi}

\subsection{Auxiliary results for LSI}

We recall the following decomposition result for R\'enyi divergence. 

\begin{lemma}[{\cite[Lemma~7]{VW19}}]\label{Lem:RenyiDecomp}
Let $q > 1$.
For all probability distributions $\rho$, $\nu$, and $\nu_\epsilon$,
\begin{align*}
R_{q,\nu}(\rho) \le \left(\frac{q-\frac{1}{2}}{q-1}\right) R_{2q,\nu_\epsilon}(\rho) + R_{2q-1,\nu}(\nu_\epsilon).
\end{align*}
\end{lemma}

We recall R\'enyi divergence is preserved under bijective map. 
Here for $T \colon \R^n \to \R^n$ and a probability distribution $\rho$, the pushforward $T_\#\rho$ is the distribution of $T(x)$ when $x \sim \rho$.

\begin{lemma}[{\cite[Lemma~13]{VW19}}]\label{Lem:RenyiBij}
Let $T \colon \R^n \to \R^n$ be a differentiable bijective map.
For any probability distributions $\rho,\nu$, and for all $q > 0$,
\begin{align*}
R_{q,T_\#\nu}(T_\#\rho) = R_{q,\nu}(\rho).
\end{align*}
\end{lemma}

We recall how the LSI constant decays along Gaussian convolution. 

\begin{lemma}[{\cite[Lemma~15]{VW19}}]\label{Lem:LSIGaussianConv}
Suppose $\nu$ satisfies LSI with constant $\alpha > 0$.
For $t > 0$, the distribution $\tilde \nu_t = \nu \ast \mathcal{N}(0,\,2tI)$ satisfies LSI with constant $\big(\frac{1}{\alpha}+2t \big)^{-1}$.
\end{lemma}

We recall the formula for the decrease of R\'enyi divergence along simultaneous heat flow. 
Here
\begin{align}
F_{q,\nu}(\rho) = \E_\nu\left[\left(\frac{\rho}{\nu}\right)^q\right] = \int_{\R^n} \nu(x) \frac{\rho(x)^q}{\nu(x)^q} \, dx = \int_{\R^n} \frac{\rho(x)^q}{\nu(x)^{q-1}} dx
\end{align}
so we can write the R\'enyi divergence as $R_{q,\nu}(\rho) = \frac{1}{q-1} \log F_{q,\nu}(\rho)$, and
\begin{align}\label{Eq:InfoDef}
G_{q,\nu}(\rho) 
= \E_\nu\Big[\Big(\frac{\rho}{\nu}\Big)^q \Big\|\nabla \log \frac{\rho}{\nu} \Big\|^2\Big] 
= \E_\nu\Big[\Big(\frac{\rho}{\nu}\Big)^{q-2} \Big\|\nabla \frac{\rho}{\nu} \Big\|^2\Big] 
= \frac{4}{q^2} \E_\nu\Big[\Big\|\nabla \Big(\frac{\rho}{\nu}\Big)^{\frac{q}{2}}\Big\|^2\Big]
\end{align}
is the R\'enyi information.
Note the case $q=1$ recovers relative Fisher information: $G_{1,\nu}(\rho) = J_\nu(\rho)$.

\begin{lemma}[{\cite[Lemma~16]{VW19}}]\label{Lem:RenyiHeat}
For any probability distributions $\rho_0,\nu_0$,
and for any  $t \ge 0$, let $\rho_t = \rho_0 \ast \N(0,2tI)$ and $\nu_t = \nu_0 \ast \N(0,2tI)$.
Then for all $q > 0$,
\begin{align*} 
\frac{d}{dt} R_{q,\nu_t}(\rho_t) = -q\frac{G_{q,\nu_t}(\rho_t)}{F_{q,\nu_t}(\rho_t)}.
\end{align*}
\end{lemma}

Finally, we recall the following relation between R\'enyi information and divergence under LSI. 
Note the case $q=1$ recovers the original definition~\eqref{Eq:LSI} of LSI.

\begin{lemma}[{\cite[Lemma~5]{VW19}}]\label{Lem:RenyiLSI}
Suppose $\nu$ satisfies LSI with constant $\alpha > 0$.
Let $q \ge 1$.
For all $\rho$,
\begin{align*} 
\frac{G_{q,\nu}(\rho)}{F_{q,\nu}(\rho)} \ge \frac{2\alpha}{q^2} R_{q,\nu}(\rho).
\end{align*}
\end{lemma}

\subsection{Proof of Theorem~\ref{Thm:RenyiRate}}\label{App:ProofThmRenyiRate}

We first show R\'enyi divergence to $\nu_\epsilon$ converges exponentially fast along PLA when $\nu_\epsilon$ satisfies LSI.
The following is analogous to~\cite[Lemma~8]{VW19} for ULA.

\begin{lemma}\label{Lem:RenyiRateLSI}
Assume $\nu_\epsilon$ satisfies LSI with constant $\beta > 0$.
Assume $\nu = e^{-f}$ is $L$-smooth, and $0 < \epsilon < \min\left\{\frac{1}{L}, \frac{1}{2\beta}\right\}$.
For $q \ge 1$, along PLA,
\begin{align}\label{Eq:RenyiRateLSI}
R_{q,\nu_\epsilon}(\rho_k) \le e^{-\frac{\beta \epsilon k}{q}} R_{q,\nu_\epsilon}(\rho_0).
\end{align}
\end{lemma}
\begin{proof}
We will prove that along each step of PLA~\eqref{Eq:PLA} from $x_k \sim \rho_k$ to $x_{k+1} \sim \rho_{k+1}$, the R\'enyi divergence with respect to $\nu_\epsilon$ decreases by a constant factor:
\begin{align}\label{Eq:RenyiDecrease}
R_{q,\nu_\epsilon}(\rho_{k+1}) \le e^{-\frac{\beta \epsilon}{q}} R_{q,\nu_\epsilon}(\rho_k).
\end{align}
Iterating the bound above yields the desired claim~\eqref{Eq:RenyiRateLSI}.

We decompose each step of PLA~\eqref{Eq:PLA} into a sequence of two steps:
\begin{subequations}
\begin{align}
\tilde \rho_k &= \rho_k \ast \N(0,2\epsilon I), \label{Eq:PLA-1} \\
\rho_{k+1} &= (I + \epsilon \nabla f)^{-1}_\# \tilde \rho_k.  \label{Eq:PLA-2}
\end{align}
\end{subequations}

In the first step~\eqref{Eq:PLA-2} we convolve with a Gaussian, which is the result of evolving along the heat flow at time $\epsilon$.
For $0 \le t \le \epsilon$, let $\rho_{k,t} = \rho_k \ast \N(0,2t I)$ and $\nu_{\epsilon,t} = \nu_\epsilon \ast \N(0,2tI)$,
so $\tilde \rho_k = \rho_{k,\epsilon}$, and let $\tilde \nu_\epsilon =  \nu_{\epsilon,\epsilon}$.
By Lemma~\ref{Lem:RenyiHeat},
\begin{align*}
\frac{d}{dt} R_{q,\nu_{\epsilon,t}}(\rho_{k,t}) = -q\frac{G_{q,\nu_{\epsilon,t}}(\rho_{k,t})}{F_{q,\nu_{\epsilon,t}}(\rho_{k,t})}.
\end{align*}
Since $\nu_\epsilon$ satisfies LSI with constant $\beta$, by Lemma~\ref{Lem:LSIGaussianConv}, $\nu_{\epsilon,t}$ satisfies LSI with constant $(\frac{1}{\beta}+2t)^{-1} \ge (\frac{1}{\beta}+2\epsilon )^{-1} \ge \frac{\beta}{2}$ for $0 \le t \le \epsilon \le \frac{1}{2\beta}$.
Then by Lemma~\ref{Lem:RenyiLSI},
\begin{align*}
\frac{d}{dt} R_{q, \nu_{\epsilon,t}}(\rho_{k,t}) = -q\frac{G_{q,\nu_{\epsilon,t}}(\rho_{k,t})}{F_{q,\nu_{\epsilon,t}}(\rho_{k,t})} \le -\frac{\beta}{q} R_{q,\nu_{\epsilon,t}}(\rho_{k,t}).
\end{align*}
Integrating over $0 \le t \le \epsilon$ gives
\begin{align}\label{Eq:Renyi-Calc2}
R_{q, \tilde \nu_\epsilon}(\tilde \rho_k) = R_{q,\nu_{\epsilon,\epsilon}}(\rho_{k,\epsilon}) \le e^{-\frac{\beta \epsilon}{q}} R_{q, \nu_\epsilon}(\rho_k).
\end{align}

In the second step~\eqref{Eq:PLA-2} we apply the proximal map $T(x) = (I + \epsilon \nabla f)^{-1}(x)$.
Since $\nabla f$ is $L$-Lipschitz and $\epsilon < \frac{1}{L}$, $T$ is a bijection.
Furthermore, $\rho_{k+1} = T_\# \tilde \rho_k$ and $\nu_\epsilon = T_\# \tilde \nu_\epsilon$.
Therefore, by Lemma~\ref{Lem:RenyiBij},
\begin{align}\label{Eq:Renyi-Calc1}
R_{q,\nu_\epsilon}(\rho_{k+1}) = R_{q, T_\# \tilde \nu_\epsilon}(T_\# \tilde \rho_k) = R_{q,\tilde \nu_\epsilon}(\tilde \rho_k).
\end{align}
Combining~\eqref{Eq:Renyi-Calc2} and~\eqref{Eq:Renyi-Calc1} gives the desired inequality~\eqref{Eq:RenyiDecrease}.
\end{proof}

\begin{proof}[Proof of Theorem~\ref{Thm:RenyiRate}]
By Lemma~\ref{Lem:RenyiDecomp} and Lemma~\ref{Lem:RenyiRateLSI},
\begin{align*}
R_{q,\nu}(\rho_k) &\le \left(\frac{q-\frac{1}{2}}{q-1}\right) R_{2q,\nu_\epsilon}(\rho_k) + R_{2q-1,\nu}(\nu_\epsilon) \\
&\le \left(\frac{q-\frac{1}{2}}{q-1}\right) e^{-\frac{\beta \epsilon k}{2q}} R_{2q,\nu_\epsilon}(\rho_0) + R_{2q-1,\nu}(\nu_\epsilon).
\end{align*}
\end{proof}

\subsection{Proof of Lemma~\ref{Lem:InitRenyi}}

\begin{proof}[Proof of Lemma~\ref{Lem:InitRenyi}]
The biased limit $\nu_\epsilon = (I+\epsilon \nabla f)^{-1}_\#(\nu_\epsilon \ast \N(0,2\epsilon I))$ satisfies $\tilde \nu_\epsilon = (I + \epsilon \nabla f)_\# \nu_\epsilon = \nu_\epsilon \ast \N(0,2\epsilon I)$, so in particular $\tilde \nu_\epsilon$ is $(\frac{1}{2\epsilon},\infty)$-smooth.
Let $\tilde \rho_0 = \N(x^\ast,2\epsilon I)$ where $x^\ast$ is a stationary point of $f$ ($\nabla f(x^\ast) = 0$).
By~\cite[Lemma~4]{VW19}, we have $R_{q,\tilde \nu_\epsilon}(\tilde \rho_0) \le \tilde O(n)$ for all $q \ge 1$.
For $\epsilon < \frac{1}{L}$, $(I + \epsilon \nabla f)^{-1}$ is a bijective map.
Then by Lemma~\ref{Lem:RenyiBij}, with $\rho_0 = (I + \epsilon \nabla f)^{-1}_\# \tilde \rho_0$, we have $R_{q,\nu_\epsilon}(\rho_0) = R_{q,\tilde \nu_\epsilon}(\tilde \rho_0) \le \tilde O(n)$, as desired.
\end{proof}

\subsection{Auxiliary results for Poincar\'e inequality}

We recall the decay of Poincar\'e constant along Gaussian convolution.

\begin{lemma}[{\cite[Lemma~20]{VW19}}]\label{Lem:PGaussianConv}
Suppose $\nu$ satisfies Poincar\'e inequality with constant $\alpha > 0$.
For $t > 0$, the distribution $\tilde \nu_t = \nu \ast \mathcal{N}(0,\,2tI)$ satisfies Poincar\'e inequality with constant $\big(\frac{1}{\alpha}+2t \big)^{-1}$.
\end{lemma}

We recall the relation between R\'enyi information and divergence under Poincar\'e inequality.

\begin{lemma}[{\cite[Lemma~17]{VW19}}]\label{Lem:RenyiPI}
Suppose $\nu$ satisfies Poincar\'e inequality with constant $\alpha > 0$.
Let $q \ge 2$.
For all $\rho$,
\begin{align*}
\frac{G_{q,\nu}(\rho)}{F_{q,\nu}(\rho)} \ge \frac{4\alpha}{q^2} \left(1-e^{-R_{q,\nu}(\rho)}\right). 
\end{align*}
\end{lemma}

\subsection{Proof of Theorem~\ref{Thm:RenyiRatePoincare}}\label{App:RenyiRatePoincare}

We first show R\'enyi divergence to $\nu_\epsilon$ converges along PLA when $\nu_\epsilon$ satisfies Poincar\'e inequality, at the same speed as the continuous-time Langevin dynamics.
The following is analogous to~\cite[Lemma~18]{VW19} for ULA.

\begin{lemma}\label{Lem:RenyiRateP}
Assume $\nu_\epsilon$ satisfies Poincar\'e inequality with constant $\beta > 0$.
Assume $\nu = e^{-f}$ is $L$-smooth, and $0 < \epsilon < \min\left\{\frac{1}{L}, \frac{1}{2\beta}\right\}$.
For $q \ge 2$, along PLA~\eqref{Eq:PLA},
\begin{align}\label{Eq:RenyiRateP}
R_{q,\nu_\epsilon}(\rho_k) \le
\begin{cases}
R_{q,\nu_\epsilon}(\rho_0) -\frac{\beta \epsilon k}{q} ~~ & \text{ if } R_{q,\nu_\epsilon}(\rho_0) \ge 1 \text{ and as long as } R_{q,\nu_\epsilon}(\rho_k) \ge 1, \\
e^{-\frac{\beta \epsilon k}{q}} R_{q,\nu_\epsilon}(\rho_0) ~~ & \text{ if } R_{q,\nu_\epsilon}(\rho_0) \le 1.
\end{cases}
\end{align}
\end{lemma}
\begin{proof}
Following the proof of Lemma~\ref{Lem:RenyiRateLSI}, 
we decompose each step of PLA~\eqref{Eq:PLA} into a sequence of two steps:
\begin{subequations}
\begin{align}
\tilde \rho_k &= \rho_k \ast \N(0,2\epsilon I), \label{Eq:PLA-3} \\
\rho_{k+1} &= (I + \epsilon \nabla f)^{-1}_\# \tilde \rho_k.  \label{Eq:PLA-4}
\end{align}
\end{subequations}

The first step~\eqref{Eq:PLA-3} is convolution with a Gaussian, which is the result of evolving along the heat flow at time $\epsilon$.
For $0 \le t \le \epsilon$, let $\rho_{k,t} = \rho_k \ast \N(0,2t I)$ and $\nu_{\epsilon,t} = \nu_\epsilon \ast \N(0,2tI)$,
so $\tilde \rho_k = \rho_{k,\epsilon}$, and let $\tilde \nu_\epsilon =  \nu_{\epsilon,\epsilon}$.
Since $\nu_\epsilon$ satisfies Poincar\'e inequality with constant $\beta$, by Lemma~\ref{Lem:PGaussianConv}, $\nu_{\epsilon,t}$ satisfies Poincar\'e inequality with constant $(\frac{1}{\beta}+2t)^{-1} \ge (\frac{1}{\beta}+2\epsilon )^{-1} \ge \frac{\beta}{2}$ for $0 \le t \le \epsilon \le \frac{1}{2\beta}$.
Then by Lemma~\ref{Lem:RenyiHeat} and Lemma~\ref{Lem:RenyiPI},
\begin{align*}
\frac{d}{dt} R_{q, \nu_{\epsilon,t}}(\rho_{k,t}) = -q\frac{G_{q,\nu_{\epsilon,t}}(\rho_{k,t})}{F_{q,\nu_{\epsilon,t}}(\rho_{k,t})} \le -\frac{2\beta}{q} \left(1-e^{-R_{q,\nu_{\epsilon,t}}(\rho_{k,t})}\right).
\end{align*}
We consider two cases:
\begin{enumerate}
  \item Suppose $R_{q,\tilde \nu_\epsilon}(\tilde \rho_k) = R_{q, \nu_{\epsilon,\epsilon}}(\rho_{k,\epsilon}) \ge 1$.
  Then for $0 \le t \le \epsilon$ we have $1-e^{-R_{q,\nu_{\epsilon,t}}(\rho_{k,t})} \ge 1-e^{-1} > \frac{1}{2}$, so $\frac{d}{dt} R_{q, \nu_{\epsilon,t}}(\rho_{k,t}) \le -\frac{\beta}{q}$, which implies $R_{q,\tilde \nu_\epsilon}(\tilde \rho_k) =  R_{q, \nu_{\epsilon,\epsilon}}(\rho_{k,\epsilon}) \le R_{q,\nu_\epsilon}(\rho_k) - \frac{\beta \epsilon}{q}$.
 \item Suppose $R_{q,\nu_\epsilon}(\rho_k) \le 1$, so $R_{q,\nu_{\epsilon,t}}(\rho_{k,t}) \le 1$ and $\frac{1-e^{-R_{q,\nu_{\epsilon,t}}(\rho_{k,t})}}{R_{q,\nu_{\epsilon,t}}(\rho_{k,t})} \ge \frac{1}{1+R_{q,\nu_{\epsilon,t}}(\rho_{k,t})} \ge \frac{1}{2}$.
Then
$\frac{d}{dt} R_{q,\nu_{\epsilon,t}}(\rho_{k,t}) \le -\frac{\beta}{q} R_{q,\nu_{\epsilon,t}}(\rho_{k,t})$, which implies
$R_{q,\tilde\nu_\epsilon}(\tilde \rho_k) = R_{q,\nu_{\epsilon,\epsilon}}(\rho_{k,\epsilon}) \le e^{-\frac{\beta \epsilon}{q}} R_{q,\nu_\epsilon}(\rho_k)$.
\end{enumerate}

In the second step~\eqref{Eq:PLA-4} we apply the proximal map $T(x) = (I + \epsilon \nabla f)^{-1}(x)$, which is a bijection since $\nabla f$ is $L$-Lipschitz and $\epsilon < \frac{1}{L}$.
Then by Lemma~\ref{Lem:RenyiBij},
\begin{align}\label{Eq:Renyi-Calc3}
R_{q,\nu_\epsilon}(\rho_{k+1}) = R_{q, T_\# \tilde \nu_\epsilon}(T_\# \tilde \rho_k) = R_{q,\tilde \nu_\epsilon}(\tilde \rho_k).
\end{align}
Combining~\eqref{Eq:Renyi-Calc3} with the two cases above gives us in one step of PLA:
\begin{align*}
R_{q,\nu_\epsilon}(\rho_{k+1})
\le
\begin{cases}
R_{q,\nu_\epsilon}(\rho_k) - \frac{\beta \epsilon}{q} ~~& \text{ if } R_{q,\nu_\epsilon}(\rho_k) \ge R_{q,\nu_\epsilon}(\rho_{k+1}) \ge 1 \\
e^{-\frac{\beta \epsilon}{q}} R_{q,\nu_\epsilon}(\rho_k) & \text{ if } R_{q,\nu_\epsilon}(\rho_k) \le 1.
\end{cases}
\end{align*}
By iterating, we conclude that
\begin{align*}
R_{q,\nu_\epsilon}(\rho_k) \le
\begin{cases}
R_{q,\nu_\epsilon}(\rho_0) -\frac{\beta \epsilon k}{q} ~~ & \text{ if } R_{q,\nu_\epsilon}(\rho_0) \ge 1 \text{ and as long as } R_{q,\nu_\epsilon}(\rho_k) \ge 1, \\
e^{-\frac{\beta \epsilon k}{q}} R_{q,\nu_\epsilon}(\rho_0) ~~ & \text{ if } R_{q,\nu_\epsilon}(\rho_0) \le 1
\end{cases}
\end{align*}
as desired.
\end{proof}

\begin{proof}[Proof of Theorem~\ref{Thm:RenyiRatePoincare}]
By Lemma~\ref{Lem:RenyiRateP}, after $k_0$ iterations we have $R_{2q,\nu_\epsilon}(\rho_{k_0}) \le 1$.
Applying the second case of Lemma~\ref{Lem:RenyiRateP} starting from $k_0$ gives $R_{2q,\nu_\epsilon}(\rho_k) \le e^{-\frac{\beta \epsilon (k-k_0)}{2q}} R_{2q,\nu_\epsilon}(\rho_{k_0}) \le e^{-\frac{\beta \epsilon (k-k_0)}{2q}}$.
Then by Lemma~\ref{Lem:RenyiDecomp},
\begin{align*}
R_{q,\nu}(\rho_k) &\le \left(\frac{q-\frac{1}{2}}{q-1}\right) R_{2q,\nu_\epsilon}(\rho_k) + R_{2q-1,\nu}(\nu_\epsilon) \\
&\le \left(\frac{q-\frac{1}{2}}{q-1}\right) e^{-\frac{\beta \epsilon (k-k_0)}{2q}} + R_{2q-1,\nu}(\nu_\epsilon)
\end{align*}
as desired.
\end{proof}

\section{Discussion}
\label{Sec:Disc}

In this paper we study the Proximal Langevin Algorithm (PLA) for sampling in $\R^n$ under isoperimetry: log-Sobolev inequality (LSI) or Poincar\'e inequality.
We prove an iteration complexity in KL divergence under LSI and third-order smoothness that matches the fastest known rate for sampling under LSI, with a better dependence on the LSI constant.
We also prove iteration complexities in R\'enyi divergence assuming the biased limit satisfies either LSI or Poincar\'e inequality;
the iteration complexity under Poincar\'e is a factor of $n$ larger than the complexity under LSI.

There are many directions for further study.
Our results assume second or third-order smoothness; it is interesting to study the convergence of PLA under weaker smoothness assumptions.
We can try to bound the bias of PLA in R\'enyi divergence.
We can try to extend our results for approximate proximal solvers, for example using optimistic or extra-gradient methods.
We can study whether proximal versions of other algorithms, such as the underdamped Langevin dynamics, have faster convergence.
It is also interesting to study whether symmetrized methods~\cite{W18} have smaller bias.
Another intriguing question is how to perform affine-invariant sampling in discrete time.

\paragraph{Acknowledgments.}
The author thanks Santosh Vempala for valuable and insightful discussions.

{\small
\bibliographystyle{plain}
\bibliography{draft2_v1_arxiv_v1.bbl}
}

\end{document}